\documentclass{article}

    \PassOptionsToPackage{numbers}{natbib}

\usepackage[preprint]{neurips_2023}
\usepackage{svg}
\usepackage{graphicx}
\usepackage{chngcntr}
\usepackage{amsthm}

\usepackage[utf8]{inputenc} 
\usepackage[T1]{fontenc}    
\usepackage{hyperref}       
\usepackage{url}            
\usepackage{booktabs}       
\usepackage{amsfonts}       
\usepackage{nicefrac}       
\usepackage{microtype}      
\usepackage{xcolor}         
\usepackage{amsmath}        

\newcommand{\innerproduct}[2]{\langle #1, #2 \rangle}
\title{Utilizing Free Clients in Federated Learning for Focused Model Enhancement}

\newcounter{thmlemma}

\newtheorem{theorem}[thmlemma]{Theorem}

\newtheorem{assumption}{Assumption}

\newtheorem{lemma}[thmlemma]{Lemma}

\newtheorem{defn}{Definition}[section] 

\author{%
  Aditya Narayan Ravi, Ilan Shomorony \\
  Department of Electrical and Computer Engineering\\
  University of Illinois Urbana-Champaign\\
  \texttt{anravi2@illinois.edu} \\ \texttt{ilans@illinois.edu} \\
}

\newcommand{\w}{{\mathbf w}}
\newcommand{\vi}{{\mathbf v}}
\newcommand{\g}{{\mathbf g}}
\newcommand{\Prio}{{\mathcal{P}}}
\begin{document}

\maketitle

\begin{abstract}
 Federated Learning (FL) is a distributed machine learning approach to learn models on decentralized heterogeneous data, without the need for clients to share their data.  
 Many existing FL approaches assume that all clients have equal importance and construct a global objective based on all clients. 
 We consider a version of FL we call Prioritized FL, where the goal is to learn a weighted mean objective of a subset of clients, designated as priority clients.
 An important question arises: How do we choose and incentivize well-aligned non-priority clients to participate in the federation, while discarding misaligned clients?
 We present FedALIGN (Federated Adaptive Learning with Inclusion of Global Needs) to address this challenge. The algorithm employs a matching strategy that chooses non-priority clients based on how similar the model's loss is on their data compared to the global data, thereby ensuring the use of non-priority client gradients only when it is beneficial for priority clients. This approach ensures mutual benefits as non-priority clients are motivated to join when the model performs satisfactorily on their data, and priority clients can utilize their updates and computational resources when their goals align.
 We present a convergence analysis that quantifies the trade-off between client selection and speed of convergence. Our algorithm 
 shows faster convergence and higher test accuracy 
 than baselines for various synthetic and benchmark datasets.
\end{abstract}

\section{Introduction}\label{sec:introduction}

Federated Learning (FL) \cite{pmlr-v54-mcmahan17a, FedSurvey21} is a distributed learning setting 
where a set of clients, with highly heterogeneous data, limited computation, and communication capabilities, are connected to a central server.
For $N$ clients, the goal of FL 
is for the central server to learn a model that minimizes
\begin{align}\label{eq:FlOG}
\textstyle
    F(\w) := \sum_{k=1}^N p_kF_k(\w),
\end{align}
where $\w \in \mathbb{R}^m$ is the model parameter. 
The objective function $F_k(\w)$ is the local objective function of the $k$th client, and $p_k$ is the fraction of data at the $k$th client. 
A key constraint of FL is that the server wishes to minimize $F(\w)$ without directly observing the data at each of the $N$ clients.

A popular method to solve \eqref{eq:FlOG} in the FL setting is Federated Averaging (FedAvg) \cite{BrendanMcMahan2017}. 
In FedAvg, $E$ iterations of Stochastic Gradient Descent (SGD) are performed locally at the clients before each communication round. 
After the $t$th local iteration,
if $t$ corresponds to a communication round (i.e., $t \pmod E = 0$),
a participating client $k$ sends the updated model parameter $\w_t^k$ to the server,
which aggregates them as $\w^t = \sum p_k \w_t^k$, updates the global model, and shares it with all clients.

The data across the clients are generally heterogeneous, in which case the model updates from the different clients are not well aligned. In \cite{li2020convergence}, the authors use a heterogeneity measure $\Gamma$, given by
\begin{align}
\textstyle \label{eq:gamma}
    \Gamma := F^{*} - \sum_{k=1}^N p_k F_k^{*},
\end{align}
which captures the gap between the global objective's minimum $F^{*}$ and a weighted sum of the minima $F_k^*$ of each local objective $F_k$. 
A convergence analysis for FedAvg (in the strongly convex case~\cite{li2020convergence}) implies that the expected convergence error $E[F{(\w_T)}] - F^*$ at time step $T$ scales as
\begin{align}\label{eq:fedavgconv}
    E[F(\w_T)] - F^* \sim (C + \Gamma)/T,
\end{align}
where $C$ is a constant. 
Intuitively, a more severe misalignment  between the local and global objectives leads to a larger $\Gamma$ and to worse training performance.
In particular, this suggests that FL systems that operate with a set of clients with reasonably aligned objectives may perform better.




\textbf{Prioritizing Clients in the Model.} 
The idea of selecting a subset of clients 
to include in each communication round in order to tackle heterogeneity is not new~\cite{Fu2022,karimireddy2021scaffold,Li2018}. 
In addition to reducing heterogeneity, 
client selection strategies have been proposed in the FL literature to
accelerate convergence rates \cite{Cho2020,Chen2020}, promote fairness \cite{fairnessSel}, and bolster robustness \cite{FastCon}. These methods prioritize certain clients based on the task at hand, during the aggregation process.


In various real-world applications, some clients are inherently prioritized, compared to others.
In subscription-based services, for instance, it is often the case that some clients have a paid subscription, while others use a free version of the service.
One example is online streaming services, which may provide perks such as ad-free content only to paid subscribers.
In such cases, the central server may want to optimize a model (e.g., a recommendation system) that prioritizes the paid subscribers, while still taking advantage of the data provided by all clients.
Another example are internet service providers, which have a set of paying clients, but may also provide free-access Wi-Fi hotspots.
In such cases, the free service may be provided in exchange for some input from the clients, such as data or computational resources. 
These clients are likely to have highly heterogeneous data, 
and may not be willing to directly share their data, which fits naturally into a FL setting.

In these cases, the global objective function $F(\w)$ should only take into consideration the priority clients, 
and
it is important to understand whether the non-priority clients can be leveraged to improve the prioritized global objective.
Is it possible to incentivize 
``well-aligned'' non-priority clients to participate and contribute to model improvement, while discarding the updates from ``misaligned'' clients?
To formalize this question, we propose the Prioritized Federated Learning (PFL) setting.
We discuss additional motivating scenarios for PFL in Section~\ref{sec:relatedworks}



\textbf{Proposed Model.} In PFL, there are $N$ clients, and a subset $\Prio \subset [1:N]$ is designated as \emph{priority} clients. Similar to \eqref{eq:FlOG}, the goal is to minimize
    $F(\w) = \sum_{k \in \Prio} p_k F_k(\w)$.
The remaining clients are designated as \emph{non-priority} clients. Non-priority clients, if chosen wisely and incorporated into the learning process, can potentially accelerate the model's convergence through additional gradient steps. 
To facilitate this, we propose to devise an algorithm capable of discerning well-aligned non-priority clients and aggregating their updates along with those from the priority clients. The critical challenge here lies in finding the right balance: including more non-priority clients might expedite the convergence but could also introduce bias into the global model, especially if these clients are not well-aligned with the primary objective. Conversely, implementing more stringent selection criteria ensures only the most well-aligned clients contribute, but at the expense of faster convergence from additional updates.
Against this backdrop, our paper will address two questions that arise in this setting: 
(1) What criterion should we use to effectively select well-aligned non-priority clients without compromising the integrity of our learning objective? (2) Can we establish a theoretical framework that quantifies the trade-off between accelerated convergence and potential bias introduced by including updates from non-priority clients?
\textbf{FedALIGN and Our Contributions.} To address the questions proposed for this PFL setting, we propose FedALIGN (Federated Adaptive Learning with Inclusion of Global Needs). 
FedALIGN's strategy for non-priority client inclusion can be combined with a
variety of existing algorithms. 
In every communication round, the server transmits not only the global model but also its associated loss to all participating clients.
Priority clients invariably contribute to the global model's updates. Non-priority clients, on the other hand, participate only if their local loss is on par with or lower than the communicated global loss. Notice that this is a key step  in incentivizing client participation, since well-aligned non-priority clients only participate if they get a satisfactory model.
The server then incorporates only those updates from clients that are close to the global loss.

The core advantage of FedALIGN lies in its adaptive capacity. It exploits well-aligned non-priority clients when their contribution is beneficial, also incentivizing their participation. At the same time, it prevents potential harm to the model from misaligned clients, thereby protecting the interests of priority clients.
Theoretical analyses suggest that FedALIGN effectively navigates data heterogeneity amongst priority clients, accelerating convergence at the expense of a manageable bias term. This bias can be subsequently fine-tuned in later communication rounds to eliminate objective bias.

We present the FedALIGN algorithm as a modification of FedAvg (although it can be built upon other algorithms too, as discussed in the supplementary material). 
To analyze its performance, 
we adapt the convergence analysis from \citet{li2020convergence}. 
For FedALIGN, the convergence error is given by:
\begin{align} \label{eq:errorfedalign}
E[F(\w_T)] - F^* \sim \frac{C + \theta_T\Gamma}{T} + \rho_T.
\end{align}
The factor $\theta_T \in [0,1]$ captures the reduction in heterogeneity achieved by FedALIGN. 
The more non-priority clients are included in all communication rounds till $T$, the lower the value of $\theta_T$. This is balanced with a bias term $\rho_T$, which captures the misalignment in the selected non-priority clients.

The terms $\theta_T$ and $\rho_T$ exhibit an inverse relationship and depend on the desired proximity between global and local losses. 
Specifically when $\theta_T = 1$ and $\rho_T = 0$ (we show in Section~\ref{sec:convergence} that this choice achieved by ignoring all non-priority clients), we recover \eqref{eq:fedavgconv}, the convergence rate for FedAvg.
The tradeoff between $\theta_T$ and $\rho_T$ enables us to fine-tune these terms, favoring a rapid convergence initially, and gradually adjusting the bias as needed.
Indeed experiments implementing FedALIGN shows that exploiting this trade-off leads to faster convergence and better performance over baseline algorithms, both on synthetic data and many benchmark datasets (FMNIST, EMNIST and CIFAR10). 


\textbf{Prioritization and Fairness in AI.} 
While the prioritization of a subset of clients may at first seem antithetical to the idea of fairness in AI, the setting we explore seeks a symbiosis between the priority clients and the non-priority clients.
The participation of a broad range of clients, each with their unique data, contributes to building robust, versatile, and more accurate models. 
This, in turn, improves the quality of services for all clients, including the priority ones, who benefit from the enhanced models, and the non-priority ones who gain access to a model trained on high-quality data and powerful computational resources.

\section{Problem Setting and Related Works}
\label{sec:relatedworks}

In PFL, we assume that there are $N$ clients, and client $k$ has the local dataset $\mathcal{B}_k$, with $|\mathcal{B}_k| = D_k$. 
The clients are split into two sets: priority clients $\Prio \subset [1:N]$ and non-priority clients. 
All clients are connected to a central server, 
who wishes to minimize the objective
\begin{align}
\textstyle    \sum_{k \in \Prio} p_k F_i(\w),
\end{align}
where $p_k := D_k/\sum_{i\in \Prio}D_i$ is the fraction of data at the $k$th client 
and $F_k(\w) := \frac{1}{D_k}\sum_{\xi \in \mathcal{B}_k} f(\w;\xi)$ is the local objective function of the  $k$th client.
Note that since we normalize the data with the total data possessed only by the priority clients, $\sum_{k \in \Prio}p_k = 1$, but $\sum_{k=1}^n p_k \neq 1$  in general.
We denote $t$ to indicate the number of local updates. In each communication round, each client that participates in that round, independently runs $E$ local epochs. Aggregation happens when $t\pmod E = 0$.


\textbf{Other Settings with Client Prioritization.} 
The need to prioritize a subset of clients may arise in contexts other than subscription services discussed in Section~\ref{sec:introduction}. 
For instance, existing FL settings that focus on optimizing objectives defined by their clients may face challenges when integrating additional clients. While these new clients may be valuable to the FL setting, their inclusion often necessitates a shift in the learning objective. FedALIGN provides a partial solution to this predicament, enabling the inclusion of well-aligned clients without necessitating alterations to the learning objective.

Another context that naturally leads to prioritization of some clients, is when some clients are slow, unreliable, or resource-poor. This leads
to the ``straggler effect'', which can significantly hamper the overall learning process~\cite{Wang2020}. The straggler effect occurs because the global model update must wait until slower clients complete their local computations, and the overall update time is 
when the overall computation time is 
dictated by the slowest participants. 
In \cite{Wang2020}, the authors use a weighting rule that prioritizes faster clients to perform more updates than the stragglers.
Further some consider penalizing slower clients by excluding their updates, after a period of time to improve aggregation speeds \cite{hard2019federated,NishioEdge}.
In many scenarios it is usually known which clients are likely to be slow. 
An example is households with multiple smart devices, where models are trained with a few powerful devices like smart phones and laptops, with additional less powerful devices like IOT devices (with potentially different data) also available for computation. 
In this case it may be fruitful to exclude stragglers from the global objective and only include their updates if they are well-aligned to this objective. 
Moreover in our supplementary material we extend our convergence analysis to the case where non-priority clients are free to participate in aggregation as they want to, naturally modelling stragglers who may only be able to provide updates in a few communication rounds. 

\textbf{Selfish Federated Learning.}
To the best of our knowledge, only one the work in~\cite{For2022} has explored a model similar to ours, which they call Selfish Federated Learning.
In the Selfish paradigm, it is assumed that non-priority clients always contribute to the federation irrespective of what global model they receive, an assumption that may not hold true in practice, since clients need to be incentivized to participate.
Secondly, the previous model assumes that the number of priority clients are much smaller than the total clients. This assumption can be overly restrictive, limiting the model's applicability in varied scenarios. Furthermore, their algorithm exhibits convergence only under highly stringent learning rates, thereby further confining its practical utility.
In terms of client participation, the model necessitates interaction with all non-priority clients in every communication round, which is not realistic. Lastly, it presumes identical data distribution among priority clients, which essentially aligns it with the framework of Personalized FL, rendering the premise of their approach somewhat trivial.
Our work, in contrast, imposes no such assumptions, but instead looks at a more pragmatic scenario as outlined in the introduction. While an other study \cite{Chayti2021} has investigated scenarios where clients not directly involved in the global objective can still contribute to its improvement, these works necessitate full participation. 
These approaches are incompatible with federated learning contexts, where client incentivization is crucial. Our model and algorithm naturally addresses these challenges. 

\textbf{Personalized FL.} 
Various FL strategies have been explored to create personalized models. One common approach involves adapting a global FL model to each local client's specific requirements \cite{HaoPFL,Yang22,chai2020tifl,smith2018federated}, while another leverages inter-client relationships to fine-tune personalized models, often achieved by introducing adjustable regularization terms into each client's local objective \cite{Kairouz2019,karimireddy2021scaffold,cheng2021federated}. In these scenarios, a universally applicable model is usually maintained, which is subsequently refined to suit individual clients' local data.
Furthermore, clustering techniques \cite{sattler2019clustered} have been employed, wherein clients are categorized based on the similarity of their data, leading to the development of models tailored to each cluster's unique characteristics. The end goal remains the same: the customization of models for each participating client during the aggregation phase. For a comprehensive review of personalized FL, we refer to the survey in \cite{Tan2022}.
Our approach deviates from these strategies in two critical ways: (1) Our global objective is constructed solely from a pre-selected subset of clients, who might exhibit substantial heterogeneity, and (2) Our training does not explicitly weigh the importance of non-priority clients. Instead, we train a single model that inherently benefits non-priority clients that are well-aligned. This strategy diverges from the Personalized FL paradigm, where every client's specificities are given significant consideration. 

\textbf{Tackling Heterogeneity in FL.} Many methods aim to tackle data heterogeneity in FL. FedProx \cite{Kairouz2019} adds proximal terms to local objectives and SCAFFOLD \cite{karimireddy2021scaffold} tackles client drift. 
There are many other relevant works which aim to tackle heterogeneity \cite{AdaFedReddi,haddadpour2019convergence,KhaledHetero19,stich2021errorfeedback,woodworth2020local,koloskova2021unified,Zhang_2021,pathak2020fedsplit}.
These methods aims to tune a global model to work well on all local clients. FedALIGN also targets heterogeneity, but  specifically tries to reduce the heterogeneity for priority clients by leveraging non-priority clients. 
The selection approach of FedALIGN may be combined 
with other algorithms to improve their performance in this setting.

\textbf{Client Utility to Weigh Updates.}
Several existing methods employ sample re-weighting strategies based on individual client utility. Some approaches define client utility based on the norm of gradients \cite{Zeng21}, while others leverage local loss values as the utility measure \cite{Chen2020}. One particular study biases weights in favor of clients with higher losses for faster convergence towards the global objective, albeit at the expense of an enduring bias term \cite{Cho2020}.
Alternatively, some approaches increase weights on clients with lower losses, mitigating the risk of clients with noisy data adversely influencing the global model \cite{pmlr-v97-ghadikolaei19a}. Our approach diverges from these strategies. We employ a loss matching strategy, wherein non-priority clients only share models when the received global model yields sufficiently low loss on their local dataset. The server then aggregates these models, contingent upon their proximity to the global loss. Our setting naturally allows for such a strategy, since we have separated the clients we intend to prioritize while learning the model.
Furthermore, additional studies have examined client re-weighting based on specific considerations like variance in local updates \cite{Wang2020}. For a more extensive overview of client selection and update, we refer the reader to the survey in \cite{Fu2022}. 


\section{FedALIGN and Convergence Analysis}
\label{sec:convergence}
In this section we describe FedALIGN and prove its convergence guarantee. 
Our results are presented for the full client participation setting, though the analysis can be easily extended to the case where a subset of clients (both priority and non-priority) are picked uniformly at random. Further we can extend the result for a more general case, where non-priority clients follow an arbitrary participation rule or are given the choice to participate in any given communication round. 
This could be relevant in the stragglers setting, for example.
These cases are considered in the supplementary material.

\subsection{FedALIGN}
FedALIGN introduces a simple decision rule to select clients based on the relative discrepancy between local and global loss. The non-priority clients are conditionally included in the aggregation process during a communication round corresponding to the $t$th local round, if for the provided global model $\w_t$, and a chosen threshold $\epsilon_t$ (which can, in general be tuned in every communication round), the absolute value of the difference between a non-priority client's local loss and the global loss is below $\epsilon_t$; i.e., $|F(\w_t) - F_k(\w_t)| < \epsilon_t$.
This method ensures that updates contributing to the global model do not diverge significantly from the overarching learning objective, which is determined by the priority clients. By strategically aligning the updates from non-priority clients with the global learning objective, FedALIGN can effectively harness the collective learning power of the entire network, while maintaining focus on the primary learning task.

In practice, during a communication round, along with the model parameter $\w_t$, the server sends the value of accuracy on the global data. The non-priority clients send back updates, only if the accuracy of the received model on it's local data is high enough; i.e., only if $F_k(\w_t) \leq F(\w_t) + \epsilon_t$. This way the the non-priority client is incentivized to participate in the aggregation in that round, if it obtains a good model that produces a high accuracy on its local data. The server then decides to aggregate it to the global model if it is well aligned, based on how close the accuracies are; i.e., if the full condition $|F(\w_t) - F_k(\w_t)| \leq \epsilon_t$ is satisfied.

\textbf{On the willingness of non-priority clients to participate.} We argue that the receipt of a well-performing global model provides a compelling incentive for non-priority clients to contribute to the aggregation. In our supplementary material, we experimentally show numerous instances where the globally trained model outperforms locally trained counterparts. During instances where the received model does not surpass the performance of a locally trained model, primarily due to the latter's fine-tuned adaptation to specific client characteristics, we posit that globally trained models help refine local models. This concept of leveraging external data for local fine-tuning has been the subject of extensive research in Personalized FL \cite{Tan2022}. Furthermore, we assert that access to a global model that demonstrates satisfactory performance on local client data is invariably beneficial. It contributes to the suite of tools available to the local client to potentially improve overall performance.


\subsection{Convergence Analysis}

Before stating the main theoretical result on the convergence of FedALIGN, we present key assumptions and notation.
We make the following assumptions on local objective functions $F_1,\dots,F_n$.
Assumptions~\ref{assump:smooth} and \ref{assump:convex} are standard in optimization,
while Assumptions~\ref{assump:unbiased} and \ref{assump:nodiv} are standard in distributed learning and FL literature~ \cite{li2020convergence,Cho2020,zhang2013comunicationefficient,stich2019local}.

\begin{assumption}\label{assump:smooth} (Smoothness)
For each $f\in\{F_1,\dots,F_N\}$, $f$ is $L$-smooth, that is, for all $\mathbf{u}$ and $\mathbf{v}$, $f(\mathbf{v}) \leq f(\mathbf{u}) + (\mathbf{u} - \mathbf{v})^{T}\nabla f(\mathbf{v}) + \frac{L}{2}\|\mathbf{u} - \mathbf{v}\|^2$.
\end{assumption}

\begin{assumption}\label{assump:convex}  (Convexity)
For each $f\in\{F_1,\dots,F_N\}$, $f$ is $\mu$-strongly convex, that is, for all $\mathbf{u}$ and $\mathbf{v}$, $f(\mathbf{v}) \geq f(\mathbf{u}) + (\mathbf{u} - \mathbf{v})^{T}\nabla f(\mathbf{v}) + \frac{\mu}{2}\|\mathbf{u} - \mathbf{v}\|^2$.
\end{assumption}

\begin{assumption}\label{assump:unbiased} (Unbiased and Bounded Variance of Gradients)
For a mini-batch $\xi_k$ randomly sampled from $B_k$ from user $k$, the SGD is unbiased; i.e., $\mathbb{E}[\nabla F_k(\w_k, \xi_k)] = \nabla F_k(w_k)$. Also, the variance of random gradients is bounded: $\mathbb{E}[\lVert \nabla F_k(\w_k, \xi_k) - \nabla F_k(\w_k) \rVert^2] \leq \sigma^2$ for $k = 1, \dots, N$.
\end{assumption}
Note that in general we can assume that each client has a distinct variance bound $\sigma_k$. We consider this in the supplementary material, while making a simplifying assumption here for brevity.

\begin{assumption}\label{assump:nodiv} (Non-Divergent Gradients)
The stochastic gradient's expected squared norm is uniformly bounded, i.e., $\mathbb{E}[\lVert \nabla F_k(\w_k, \xi_k) \rVert^2] \leq G^2$ for $k = 1, \ldots, N$.
\end{assumption}



As in \eqref{eq:gamma}, we define $\Gamma := F^{*} - \sum_{k \in \Prio} p_k F_k^{*}$ to be the heterogeneity parameter, but here only defined for the priority clients $\Prio$.
A larger $\Gamma$ implies a higher degree of heterogeneity. 
When $\Gamma = 0$, the local objectives and global objectives are consistent.
Recall that the index $t$ counts the number of local updates, and each client runs $E$ local updates before sending the updates when $t\pmod E = 0$.

We also define $\Gamma_k:= F_k(\w^*) - F_k^*$. Here, $\Gamma_k$ is the objective gap between the optimal global model and the optimal local model.
For a time $t$, we also define
$$\tau(t) := \max \{ t^{\prime} \,:\, t^{\prime}\pmod E = 0,  t^{\prime}\leq t\}$$
as the time at which the last communication round occurred.




Here we present the convergence result for FedALIGN in the full participation scenario and compare it to existing results. 
Suppose FedALIGN runs for $T$ iterations, for a $T$ such that $T \pmod E = 0$ and outputs $\w_T$ as the global model.

\begin{theorem}\label{thm:conv} (Convergence) Suppose that Assumptions~\ref{assump:smooth}, \ref{assump:convex}, 
\ref{assump:unbiased}, and \ref{assump:nodiv} hold, and we set a decaying learning rate $\eta_t = \frac{2}{\mu(t + \gamma)}$, where $\gamma = \max{\left\{\frac{8L}{\mu},E\right\}}$ 
and any $\epsilon_t \geq 0$.
The expected error following $T$ total local iterations (equivalent to $ T/E$ 
communication rounds) for \textnormal{FedALIGN}
under full device participation satisfies 
\begin{align}
    \mathbb E[F(\w_{T})] - F^* \leq \frac{1}{T+\gamma} \left(C_1 + C_2 \theta_T\Gamma\right)  + \rho_T,
\end{align}
where $\theta_T \in [0,1]$ and $\rho_T$ can be adjusted by selecting an appropriate $\epsilon_t$ 
in each time step ($\rho_T$ and $\theta_T$ are defined after the theorem statement). Here, $C_1$ and $C_2$ are constants given by
\begin{align}
    C_1 := \frac{2L}{\mu^2}\left(\sigma^2 + 8(E-1)^2G^2\right) + \frac{4L^2}{\mu}\|\w_0 - \w^*\|^2, \quad 
    C_2 := \frac{12L^2}{\mu^2}. \nonumber
\end{align}
\end{theorem}
Here the terms $\theta_T$ and $\rho_T$ are defined as 
\begin{align}
     \theta_T = \frac{1}{T + \gamma-2} \sum_{i=1}^{T-1}\mathbb E \left[\frac{1}{1 + \sum_{k \notin \Prio}p_k I_{k,\tau{(i)}}}\right],
\end{align}
and 
\begin{align}
 \rho_T := \frac{2L}{\mu(T+\gamma-2)}\sum_{i=1}^{T-1}\mathbb E\left[\frac{\sum_{k\notin P}p_kI_{k,\tau(i)}\Gamma_k}{{1 + \sum_{k\notin \Prio}p_k I_{k,\tau(i)}}}\right].
\end{align}

This theorem is proved in the supplementary material. We now define and explain some of the important constants we introduce in the result.

\textbf{Alignment of non-priority clients $\theta_T$.} The expression for $\theta_T$ is
\begin{align}
    \theta_T &:= \frac{1}{T + \gamma-2} \sum_{i=1}^{T-1}\mathbb E \left[\frac{1}{1 + \sum_{k \notin \Prio}p_k I_{k,\tau{(i)}}}\right],
\end{align}
where $I_{k,t} := \mathbf{1}{\{|F_{k} \left(\w_t^{k}\right) - F\left(\w_t^{k}\right)| < \epsilon_t\}}$ is an indicator random variable for  whether a non-priority client is included in a given communication round. The alignment of non-priority clients, $\theta_T$, 
effectively measures the average inclusion of non-priority clients over $T-1$ time steps. 
Lower values of $\theta_T$ correspond to a greater number of clients being included in aggregation. Specifically the convergence rate improves to $O(\theta_T/T)$ in each time step. This kind of improvement due to client selection strategies has been considered before \cite{Cho2020}.

\textbf{Tunable bias $\rho_T$.} The expression for $\rho_T$ is
\begin{align}
    \rho_T &:= \frac{2L}{\mu(T+\gamma-2)}\sum_{i=1}^{T-1}\mathbb E\left[\frac{\sum_{k\notin P}p_kI_{k,\tau(i)}\Gamma_k}{{1 + \sum_{k\notin \Prio}p_k I_{k,\tau(i)}}}\right].
\end{align}
This term is an average aggregation of the bias introduced in the first $T-1$ time steps and the added bias at time step $T$. 

The parameters $\theta_T$ and $\rho_T$
effect the convergence of the learning model under consideration, in opposite ways. The choice of $\epsilon_t$ at each time step plays a crucial role in controlling this balance. If $\epsilon_t$ is set too low, it risks under-utilizing the potential of well-aligned clients, thereby causing a slower rate of convergence. On the other hand, if $\epsilon_t$ is set too high, it introduces a high bias in the aggregated model updates, which can negatively impact the model's performance.

\textbf{Fine-tuning $\epsilon_t$.} The above dichotomy highlights the fine-tune \cite{cheng2021federated} aspect of the FedALIGN algorithm. It allows for fine-tuning of the bias through the careful selection of $\epsilon_t$ at each time step. We can thus set $\epsilon_t$ such that it is gradually reduced in the later rounds. This strategy ensures that the algorithm initially benefits from the contributions of a larger number of clients (albeit potentially introducing some bias), but as the rounds progress and the model is closer to convergence, the bias can gradually be eliminated by reducing $\epsilon_t$.
In this way, the fine-tune effect provides a mechanism for controlling the trade-off between bias and convergence speed. It enables the algorithm to leverage the power of well-aligned clients to expedite convergence initially, and then gradually reduces the bias to ensure accurate model learning. 


\textbf{Consistency of our result with FedAvg on priority clients}. 
If we set $\epsilon_t = 0$ for all time steps in FedALIGN,
$\theta_T$ (the alignment of non-priority clients) becomes $1$, implying that all priority clients are included in the aggregation. At the same time, $\rho_T$ (the tunable bias) becomes $0$, indicating that no bias is introduced in the model updates. 
Under these conditions, FedALIGN is identical to FedAvg on just the priority clients, and the convergence bound is precisely the result in \cite{li2020convergence}. It shows how FedALIGN builds upon FedAvg by introducing additional flexibility and control through the $\epsilon_t$ parameter, which allows us to manage the trade-off between bias and speed of convergence. 
We also highlight that we can modify other algorithms like FedPROX \cite{Kairouz2019} to obtain a similar control. We show experimental results in our supplementary material to validate this.






\section{Experiments}
\label{sec:experiments}

We conduct experiments on various benchmark datasets: FMNIST, \emph{balanced} EMNIST and CIFAR-10 \cite{Xiao17fmnist}. The datasets were preprocessed following the methods developed in \cite{BrendanMcMahan2017}. For FMNIST we use a logistic regression model, for EMNIST we use a two-layer neural network and for the CIFAR-10 dataset we use a Convolutional Neural Network. The dataset is generated by distributing uni-class shards to each client, following the steps in \cite{BrendanMcMahan2017}.
More details about the architectures are deferred to the supplementary material.
In addition, in order to assess the effect of noisy and irrelevant data in the non-priority clients, we modify the \textsc{Synth}$(\alpha,\beta)$ dataset considered in \cite{Li2018}. We use a logistic regression model for this experiment.

We evaluate FedALIGN (built on FedAvg) against two baselines: (1) FedAvg on only the priority clients and (2) FedAvg on all clients. Further in our supplementary material we consider FedALIGN built on FedPROX \cite{Li2018}. We consider the full participation case, where all clients are sent a global model by the server and consider the case of partial participation (where only a randomly sampled subset of clients are chosen in each round) in the supplementary material. 

\textbf{Experiment results on benchmark data}. In the design of our experiment, we considered a total of $N = 60$ clients, with $2$ clients as priority clients and the remaining as non-priority clients. The rationale behind this choice was to create significant heterogeneity between the priority and non-priority client groups, while ensuring diversity within the non-priority group. Specifically, the skewed assignment of classes leads to this significant heterogeneity.
We further explore alternative configurations of priority client selections with varying data fractions in the supplementary material.
For the learning rate ($\eta$), we adopted a value of $0.1$ for both FMNIST and EMNIST datasets, whereas for CIFAR-10, we utilized a lower learning rate of $0.01$. The selection of these learning rates was made after conducting a grid search of different learning rates on small sets of the data. The selected learning rates resulted in stable convergence and exhibited the best performance.
All experiments were conducted with 5 different seeds for randomness, setting the local epoch $E = 5$. We expand on the impact of varying epoch choices in the supplementary material. The initial $10 \%$ of communication rounds were dedicated as warm-up rounds during which only priority clients contributed to the aggregation process. This design was to ensure the development of a reasonably working model prior to the inclusion of non-priority clients.
The selection parameter $\epsilon$ was set to $0.2$, and we found that for the datasets considered, there was no need to fine-tune this parameter to $0$. We hypothesize that this may be due to the strong alignment of the clients chosen in each round.
The results depicted in Figure~\ref{fig:benchmark} show the superiority of FedALIGN over the baseline methods. This is evident in the post-warm-up phase, where there is a noticeable acceleration in convergence rates. Interestingly, it was observed that the inclusion of all clients often led to a degradation in overall performance, further affirming the importance of selecting well-aligned clients.

\begin{figure}[htb!]
\centering
\includegraphics[width=1\linewidth]{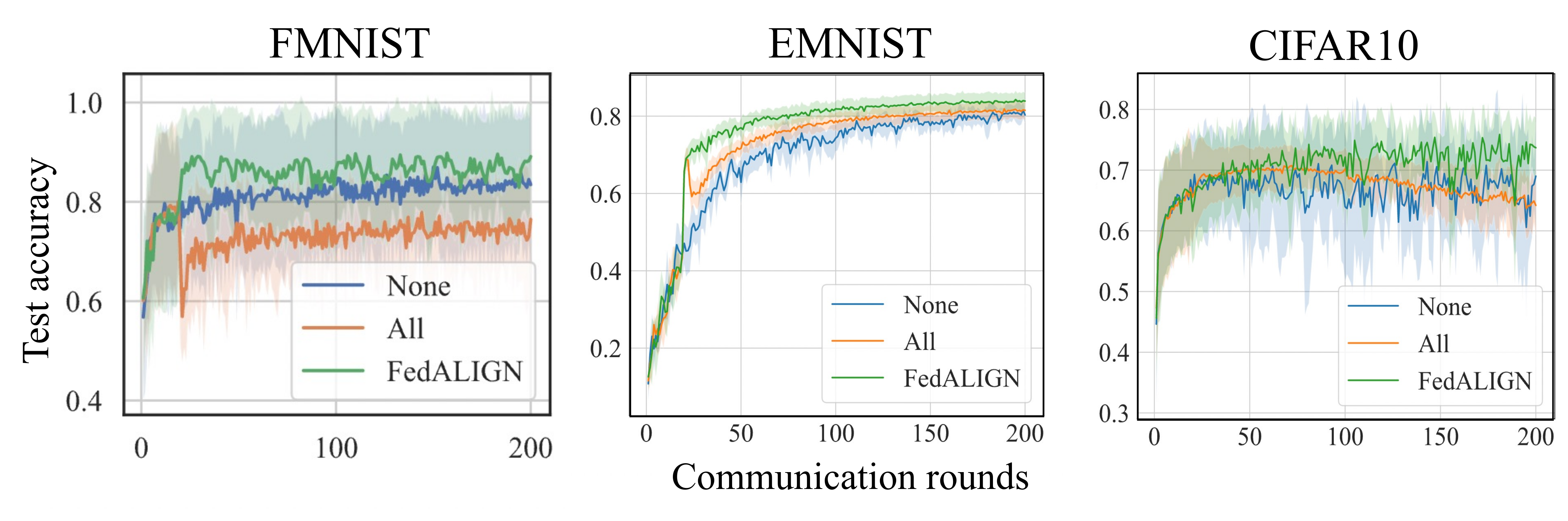}
\caption{\label{fig:benchmark}
Test accuracy for benchmark datasets, under full participation, with $2$ heterogeneous priority clients and $N = 60$ clients, with $E =5$. After the first $20$ warm-up rounds, a clear increase in convergence rate is observed for FedALIGN (green), while a deterioration is observed for All (orange) (FedAvg with all clients). FedALIGN also achieves a higher final accuracy.}
\end{figure}

\textbf{Experimental results on synthetic data.} 
We simulate scenarios to analyze the performance of non-priority clients with varying degrees of alignment. These scenarios are designed to replicate conditions where clients are affected by noisy and irrelevant data, leading to misalignment caused by factors other than the skewed class assignment that we previously considered. For this purpose, we utilize the \textsc{Synth}$(1,1)$ dataset as referred to in \cite{Li2018}.
We consider $N = 20$, with $10$ priority and $10$ non-priority clients, since \textsc{Synth}(1,1) inherently generates very heterogenous models.
 We create a global dataset and distibute this data to non-priority clients and add two forms of noise: (1) Discrepancies in label assignments for data points, and (2) the addition of irrelevant data points, which are generated by independent distributions.
We configure the noise to ensure that the non-priority clients experience varying degrees of noise, thereby influencing their alignment with the objective. The detailed methodology for the generation of this noise is provided in the supplementary materials.
We apply a logistic regression model to this data, maintaining a learning rate of $0.1$, with local models used to find the best and stable learning rates, as in the prior experiments. 
We consider $3$ different average noise/discrepancy levels across non-priority clients: medium, low and high.
We found that $\epsilon = 0.2$ is the best choice for the selection parameter for low and medium noise, while $\epsilon = 0.4$ is the best choice for the high noise experiment.
From Figure~\ref{fig:synth}, it becomes evident that FedALIGN consistently outperforms baselines, under this more complicated form of misalignment in the non-priotity clients. This underscores the robustness of FedALIGN even in challenging conditions. Our supplementary material considers various different choices of $\alpha,\beta$ to generate different kinds of \textsc{Synth}($\alpha,\beta$) data, as well as different skews on the number of data points and noise.
\clearpage

\begin{figure}[htb!]
\centering
\includegraphics[width=1\linewidth]{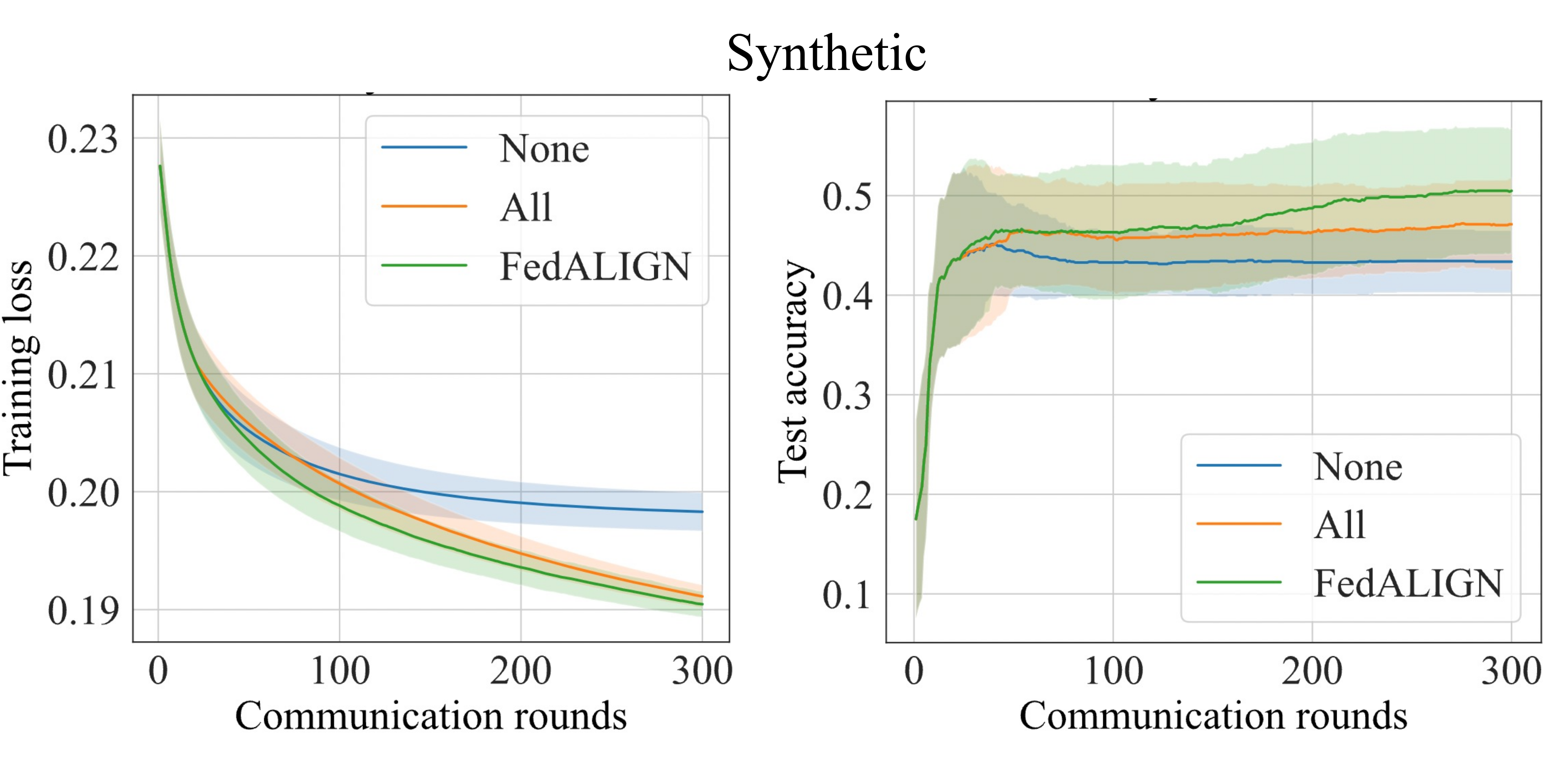}
\includegraphics[width=1\linewidth]{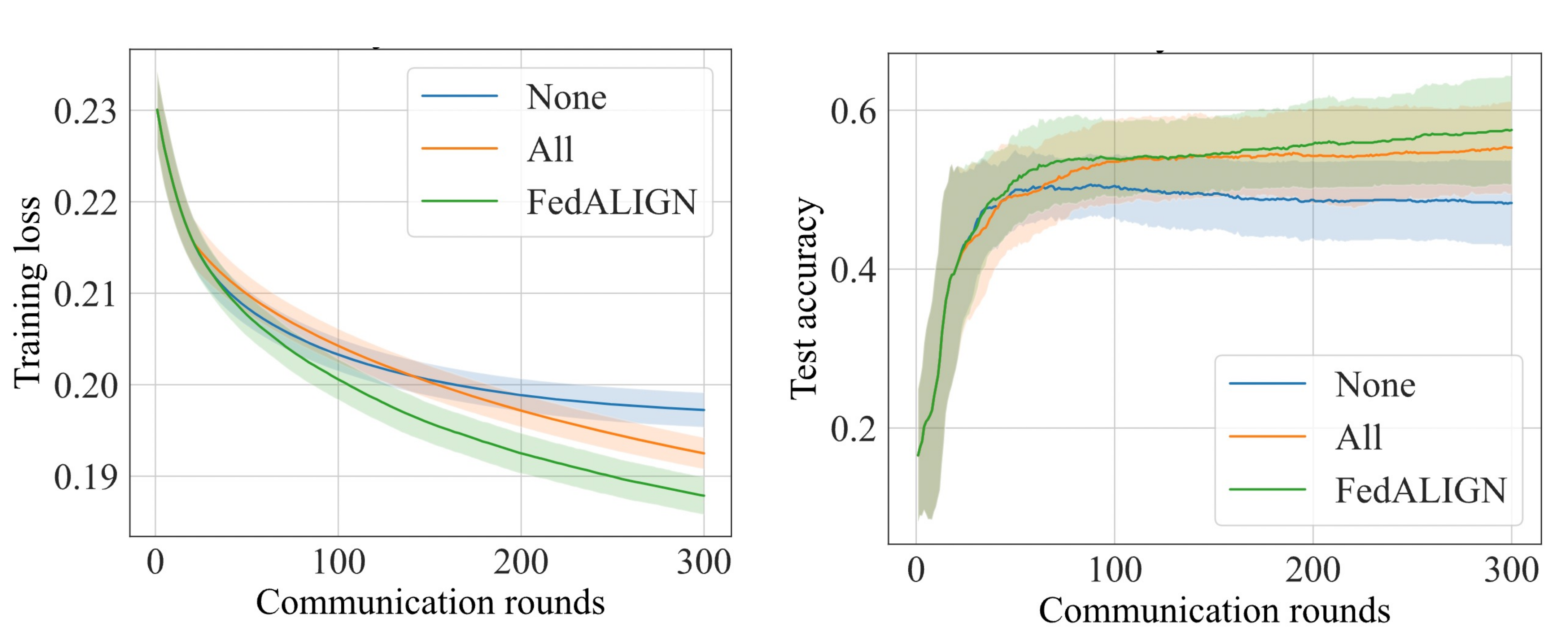}
\includegraphics[width=1\linewidth]{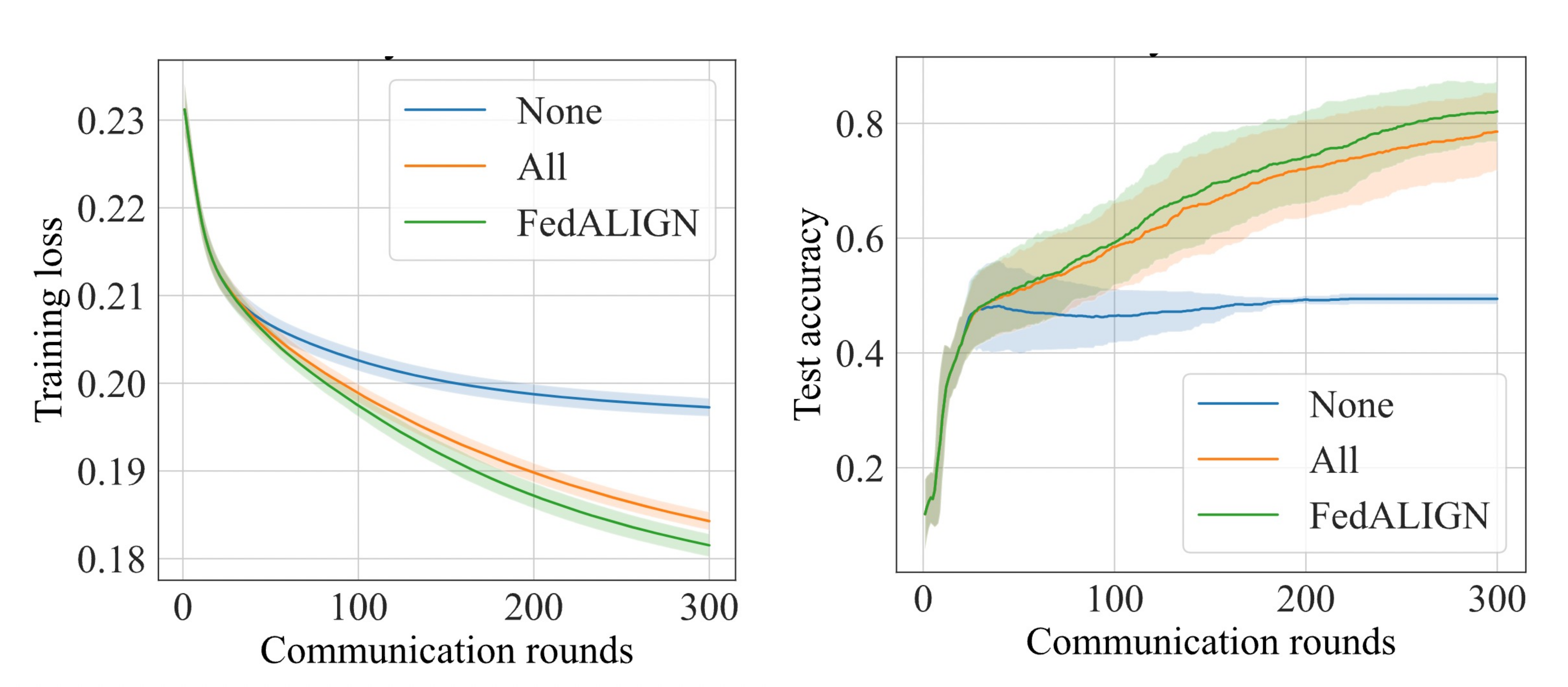}
\caption{\label{fig:synth}
Training loss and test accuracy for the synthetic dataset \textsc{Synth}(1,1), under full participation, with $10$ heterogeneous priority clients and $N = 20$ clients, with $E =5$. The first graph considers the case where there is medium noise, the second considers the low noise case and the third considers the high noise case.
After the first $20$ warm-up rounds, a clear improvement is observed for FedALIGN (green) for all these cases. The level of alignment of the non-priority clients is complicated in these experiments, where the alignment of non-priority clients in general depends on data distribution, noise and discripencies in a complex way} 
\end{figure}

\section{Concluding Remarks}
\label{sec:conclusion}

We proposed and examined Prioritized Federated Learning, a FL setting that models the real-world situation where some clients 
are given priority from the point of view of model training, but other clients may choose to participate in the federation as well.
We introduce an algorithm called FedALIGN that provides a selection rule that can be combined with FL algorithms such as FedAvg and FedPROX. 
FedALIGN smartly selects non-priority clients to help achieve the goals of priority clients. This creates a win-win situation where non-priority clients that align well with the system obtain a strong global model that can enhance their local tasks.
The analysis of convergence of FedALIGN highlights an important trade-off: using more priority clients can speed up convergence but introduces a bias. We discuss how this can be adjusted over time to eventually eliminate this bias in later communication rounds. Our extensive experiments support these findings.
Potential directions for future research include considering different criteria to decide which non-priority client updates to include in each step, and utilizing non-priority clients to align with specific priority clients' objective, rather than just the global objective. 
\medskip

\newpage

\bibliographystyle{unsrtnat}
\bibliography{main}

\newpage
\appendix
\counterwithin{thmlemma}{section}

\section{Proof of Theorem~1}

\begin{theorem}\label{thm:conv_new} (Convergence) Suppose that Assumptions~\ref{assump:smooth}, \ref{assump:convex}, 
\ref{assump:unbiased}, and \ref{assump:nodiv} hold, and we set a decaying learning rate $\eta_t = \frac{2}{\mu(t + \gamma)}$, where $\gamma = \max{\left\{\frac{8L}{\mu},E\right\}}$ 
and any $\epsilon_t \geq 0$.
The expected error following $T$ total local iterations (equivalent to $ T/E$ 
communication rounds) for \textnormal{FedALIGN}
under full device participation satisfies 
\begin{align}
    \mathbb E[F(\w_{T})] - F^* \leq \frac{1}{T+\gamma} \left(C_1 + C_2 \theta_T\Gamma\right)  + \rho_T,
\end{align}
where $\theta_T \in [0,1]$ and $\rho_T$ can be adjusted by selecting an appropriate $\epsilon_t$ 
in each time step ($\rho_T$ and $\theta_T$ are defined after the theorem statement). Here, $C_1$ and $C_2$ are constants given by
\begin{align}
    C_1 := \frac{2L}{\mu^2}\left(\sigma^2 + 8(E-1)^2G^2\right) + \frac{4L^2}{\mu}\|\w_0 - \w^*\|^2, \quad 
    C_2 := \frac{12L^2}{\mu^2}. \nonumber
\end{align}
\end{theorem}
Here the terms $\theta_T$ and $\rho_T$ are defined as 
\begin{align}
     \theta_T = \frac{1}{T + \gamma-2} \sum_{i=1}^{T-1}\mathbb E \left[\frac{1}{1 + \sum_{k \notin \Prio}p_k I_{k,\tau{(i)}}}\right],
\end{align}
and 
\begin{align} \label{eq:correct_rho}
 \rho_T := \frac{2L}{\mu(T+\gamma-2)}\sum_{i=1}^{T-1}\mathbb E\left[\frac{\sum_{k\notin P}p_kI_{k,\tau(i)}\Gamma_k}{{1 + \sum_{k\notin \Prio}p_k I_{k,\tau(i)}}}\right].
\end{align}

We analyze FedALIGN in the full device participation setting here. The proof skeleton here follows \cite{Cho2020}, but considering a selection criteria for non-priority clients in each communication round requires novel analysis.


\subsection{Some definitions and preliminaries}
First we define the \emph{re-normalized} data fraction as 
\begin{align}
p'_{k}(t) := \frac{p_k}{1 + \sum_{k \notin \Prio}p_k I_{k,\tau(t)}}.
\end{align}
Generally, $p'_k(t)$ is a random variable, dependent on $(\epsilon_{\tau(t)},\tau{(t)})$, and can assume a distinct value in each communication round. In the proof, we sometimes express $p'_{k}(t)$ as $p^{\prime}_k$, which is a minor notational liberty.

Note that our aggregation rule is as follows: During every communication round, the global model $\w_{\tau(t)}$ aggregates the weighted sum of the local models, renormalized by total data contributed by all clients included in that communication round. This update is given by
\begin{align}
    \w_{\tau(t)} &\leftarrow \frac{1}{1 + \sum_{k\not\in \Prio}p_kI_{k,\tau(t)}}\left(\sum_{k \in \Prio}p_k\w_{\tau(t)}^k + \sum_{k \notin \Prio}p_k\w_{\tau(t)}^kI_{k,\tau(t)}\right) \nonumber. 
\end{align}
This aggregated expression rewritten in terms of $p^{\prime}_k$ has the following form:
\begin{align}
    \sum_{k \in \Prio}p^{\prime}_k\w_{\tau(t)}^k + \sum_{k \notin \Prio}p^{\prime}_k\w_{\tau(t)}^kI_{k,\tau(t)}.
\end{align}
This motivates us to define a ``virtual'' sequence
\begin{align}
\w_t = \sum_{k \in \Prio}p^{\prime}_k\w_t^k + \sum_{k \notin \Prio}p^{\prime}_k\w_t^kI_{k,\tau(t)},
\end{align}
which intuitively
captures an aggregation (if it happens) in every local round. The value of $\w_t$ is indeed the aggregated global model when $t$ corresponds to a communication round.
Now, define the aggregated stochastic gradients as
\begin{align}
    \g_t = \sum_{k=1}^{N}p'_{k}\nabla F_k(\w^{k}_t,\xi_t)I_{k,\tau(t)},
\end{align}
and the aggregated gradients as 
\begin{align}
     \overline\g_t = \sum_{k=1}^{N}p'_{k}\nabla F_k(\w^{k}_t)I_{k,\tau(t)},
\end{align}
where above, for ease of analysis, we redefine $I_{k,\tau(t)}$ to be 
\begin{align}
    I_{k,\tau(t)} := \mathbf{1}{\{|F_{k} \left(\w_t^{k}\right) - F\left(\w_t^{k}\right)| < \epsilon_t, \text{ or $k$ is the index of a priority client}\}}.
\end{align}
This definition also allows us to rewrite $\w_t$ as
\begin{align}
    \w_t := \sum_{k=1}^Np^{\prime}_k\w_t^kI_{k,\tau(t)}.
\end{align}
Both definitions are equivalent, and we will use either definition as required. We also use $\mathbb E$ to mean expectation over all sources of randomness for a fixed time $t$, while $\mathbb E_{\cdot|p^{\prime}_k,I_{k,\tau(t)}\forall k}$ refers to the conditional expectation, conditioned on random variables $(p^{\prime}_k,I_{k,\tau(t)}\forall k)$.

We now state some preliminary lemmas. We state more directly relevant lemmas in the next section. 
\begin{lemma}\label{lem:l-smooth}(Gradient is $L$-close to its minima)
    If $F_k$ is $L$-smooth 
    with global minimum value $F_k^*$, then for any $\w_k$, we have
    \begin{align}
        \|\nabla F_k(\w_k)\|^2 \leq 2L (F_k(\w_k) - F_k^*).
    \end{align}
\end{lemma}
\begin{proof}
By definition of $L$-smoothness, we have
\begin{align}
   F_k^*\leq F_k(\vi) \leq F_k(\w_k) + (\vi - \w_k)^T\nabla F_k(\w_k) + \frac{L}{2}\|\vi-\w_k\|^2. \nonumber
\end{align}
Since this is true for any $\vi$, one  can set $\vi = \w_k - (1/L)\nabla F_k(\w_k)$. This gives us
\begin{align}
     F_k^* \leq  F(\w_k)  -\frac1{L}[\nabla F_k(\w_k)]^T\nabla F_k(\w_k) + \frac{1}{2L}\|\nabla F_k(\w_k)\|^2 = F_k(\w_k) - \frac{1}{2L}\|\nabla F_k(\w_k)\|^2, \nonumber
\end{align}
or
\begin{align}
    \|\nabla F_k(\w_k)\|^2 \leq 2L (F_k(\w_k) - F_k^*).
\end{align}
\end{proof}
\begin{lemma}\label{lem:aggreup}(Aggregated model is closer to $\w^*$ than local models) Given local models $\w_t^k$ and the global minima $\w^*$ of $F(\cdot)$,
\begin{align}
    \mathbb E \left[\|\w_t - \w^*\|^2\right] \leq \mathbb E \sum_{k = 1}^Np_k^{\prime} I_{k,\tau(t)} \mathbb E_{.|p_k^{\prime} I_{k,\tau(t)}}\|\w_t^k - \w^*\|^2.
\end{align}
\end{lemma}
\begin{proof}
The lemma follows because
\begin{align}
    &\mathbb E \left[\|\w_t - \w^*\|^2\right] = \mathbb E \left[\left\|\sum_{k=1}^Np_k^{\prime}I_{k,\tau(t)}\w_k^{t} - \w^*\right\|^2\right] \stackrel{(a)}{=} \mathbb E \left[\left\|\sum_{k=1}^Np_k^{\prime}I_{k,\tau(t)}(\w_k^{t} - \w^*)\right\|^2\right] \nonumber \\
    &\leq \mathbb E\sum_{k=1}^np_k^{\prime}I_{k,\tau(t)}I_{k,\tau(t)} \|\w_t^k - \w^{*}\|^2 \leq \mathbb E\sum_{k=1}^np_k^{\prime}I_{k,\tau(t)}\mathbb E_{.|p_k^{\prime} I_{k,\tau(t)}} \|\w_t^k - \w^{*}\|^2,
\end{align}
where $(a)$ is because $\sum_{k=1}^N p_k^{\prime}I_{k,\tau(t)} = 1$, from the definitions of $p_k^{\prime}$ and $I_{k,\tau(t)}$.
\end{proof}
\begin{lemma}\label{lem:avgdisc}(Average discrepancy caused by local updates.) Given Assumption~\ref{assump:nodiv}, for $\eta_{\tau(t)} \leq 2 \eta_t$ and all $t - \tau(t) \leq E - 1$, we have
\begin{align}
        \mathbb E \sum_{k=1}^N p_k^{\prime} \|\w_t^k - \w_t\|^2I_{k,\tau(t)} \leq 4\eta_t^2(E-1)^2G^2.
\end{align}
\end{lemma}
\begin{proof}
We have that
\begin{align}
    \mathbb E \sum_{k=1}^N p_k^{\prime} \|\w_t^k - \w_t\|^2I_{k,\tau(t)} &= \mathbb E \sum_{k=1}^N p_k^{\prime} \left\|\left(\w_t^k - \w_{\tau(t)}\right) -  \left(\w_t - \w_{\tau(t)}\right)\right\|^2I_{k,\tau(t)} \nonumber \\
    &\stackrel{(a)}{\leq} \mathbb E \sum_{k=1}^N p_k^{\prime} \|\w_t^k - \w_{\tau(t)}\|^2I_{k,\tau(t)} \nonumber \\
    &\stackrel{(b)}{\leq} \mathbb E \sum_{k=1}^N p_k^{\prime} I_{k,\tau(t)}\mathbb E_{\cdot|p_k^{\prime},I_{k,\tau(t)}}\left(\sum_{i = \tau(t)}^{\tau(t) + E - 1}(E-1)\eta_i^2\|\nabla F_k(\w_i^k)\|^2\right)\nonumber \\
    &\stackrel{(c)}{\leq} \mathbb E \sum_{k=1}^N p_k^{\prime}I_{k,\tau(t)}\sum_{i = \tau(t)}^{\tau(t) + E - 1} (E-1)\eta_{i}^2G^2 \nonumber \\
    &\stackrel{(d)}{\leq} \mathbb E \sum_{k=1}^N p_k^{\prime}I_{k,\tau(t)}\ (E-1)^2\eta_{\tau(t)}^2G^2 \nonumber \\
    &\stackrel{(e)}{\leq} 4\eta_t^2(E-1)^2G^2,
\end{align}
where $(a)$ is due to the definition $\w_t$ and the fact that (for vectors $\mathbf a_k$)
\begin{align*}
    \sum_{k=1}^N \|\mathbf a_k - \bar{\mathbf a}\|^2 \leq \sum_{k=1}^N\|\mathbf a_k\|^2,
\end{align*}
where $\bar{\mathbf a}$ is a weighted average of the $\mathbf a_k$'s, $(b)$ is due to the fact that when $t - \tau(t) \leq E - 1$
\begin{align*}
    \|\w_t^k - \w_{\tau(t)}\|^2 &= \left\|\sum_{i=\tau(t)}^{t}\nabla F(\w_i^k)\right\|^2 \leq \sum_{i=\tau(t)}^{t}(E-1)\eta_i^2\|\nabla F_k(\w_i^k)\|^2 \\
    &\leq \sum_{i=\tau(t)}^{\tau(t) + E - 1}(E-1)\eta_i^2\|\nabla F_k(\w_i^k)\|^2,
\end{align*}
$(c)$ is due to Assumption~\ref{assump:nodiv}, $(d)$ is because $\eta_t$ is non-increasing and $(e)$ is because $\eta_{\tau(t)} \leq 2\eta_t$.
\end{proof}
\begin{lemma}\label{lem:sgdvar}(SGD variance bound) Given Assumption~\ref{assump:unbiased},
\begin{align}
   \mathbb E \|\overline\g_t - \g_t\|^2 \leq \sigma^2.
\end{align} 
\end{lemma}
\begin{proof}
First we notice that
\begin{align}
    \mathbb E \|\overline\g_t - \g_t\|^2  & = \mathbb E \left\|\sum_{k=1}^Np_k^{\prime} I_{k,\tau(t)}(\nabla F(\w_t^k;\xi_t^k) - \nabla F(\w_t^k) )\right\|^2 \nonumber\\
    &\leq \mathbb E \sum_{k=1}^N{p_k^{\prime}}^2 I_{k,\tau(t)} \mathbb E_{\cdot|p_k^{\prime}, I_{k,\tau(t)}}\|\nabla F(\w_t^k;\xi_t^k) - \nabla F(\w_t^k)\|^2 \nonumber \\
    &\stackrel{(a)}{\leq} \sigma^2 \mathbb E\sum_{k=1}^N {p_k^{\prime}}^2 I_{k,\tau(t)} \stackrel{(b)}{\leq} \sigma^2,
\end{align}
where $(a)$ is due to Assumption~\ref{assump:unbiased} and $(b)$ is because $p_k^{\prime} \leq 1$ and $\sum_{k=1}^Np_k^{\prime}I_{k,\tau(t)} = 1$.
\end{proof}

Notice that in Lemma~\ref{lem:sgdvar}, if one instead has individual variance bounds of $\sigma_k^2$, one can also extend this lemma to obtain
\begin{align}
    \mathbb E \|\overline\g_t - \g_t\|^2 \leq \mathbb E\sum_{k=1}^N {p_k^{\prime}}^2\sigma_k^2
\end{align}
in a straightforward manner. If required, the subsequent analysis can be easily extended to this case.
\subsection{One-step SGD}
    
\begin{lemma}\label{lem:onestepsgd} (One step SGD) Assume Assumptions~\ref{assump:smooth}, \ref{assump:convex}, \ref{assump:unbiased} and \ref{assump:nodiv}. If $\eta_t \leq \frac{1}{4L}$ and $\w^{*}$ is the global optimum of objective $F(.)$, then
\begin{align}
        \mathbb{E} \|\w_{t+1} - \w^{*}\|^2  &\leq (1-\eta_t\mu)\mathbb{E}\|\w_t - \w^{*}\|^2 + \eta_t^2 B_t + \eta_t D_t,
\end{align}
where  
\begin{align}
    B_t &:= 8(E-1)^2G^2 + 6L\Gamma\mathbb{E}\left[\left(\frac{1}{1 + \sum_{k \notin \Prio}p_kI_{k,\tau(t)} }\right)\right] + \sigma^2\nonumber \\
    D_t &= 2\mathbb E\left[\frac{\sum_{k\notin P}p_kI_{k,\tau(t)}\Gamma_k}{{1 + \sum_{k\notin \Prio}p_k I_{k,\tau(t)}}}\right]\nonumber \\
    \Gamma &:= F^{*} - \sum_{k \in \Prio } p_k F_k^{*},   \Gamma_k := F_k(\w^*) - F_k^*.
\end{align}
\end{lemma}
We prove this lemma in the next section. Now define $\Delta_t = \mathbb{E} \|\w_t - \w^*\|^2.$
We assume a diminishing step-size $\eta_t =\frac{2}{\mu(t+\gamma)}$, with $\eta_1 \leq \min{\left(\frac{1}{4L},\frac{1}{\mu}\right)}$ and $\eta_{t} \leq 2\eta_{t + E}$ (The choice of $\gamma$ guarantees this condition).
We 
 can prove that $\Delta_t \leq \frac{v_t}{\gamma +t}$ by induction, where $v_t = \left(\frac{4}{\mu^2}\overline{B}_t + \frac2{\mu}\overline{D}_t(t+\gamma) +  \gamma\Delta_0\right)$, where (if we pick $\gamma = \max\{\frac{8L}{\mu},E\}$)
\begin{align}
    \overline{B}_0,\overline{B}_1 = 0 ,\overline{B}_t &= \frac{1}{t + \gamma - 2}\sum_{i=1}^{t-1}B_i\hspace{3mm}\forall t\geq 2, \nonumber \\
      \overline{D}_0,\overline{D}_1 = 0,\overline{D}_t &= \frac{1}{t + \gamma - 2}\sum_{i=1}^{t-1}D_t\hspace{3mm}\forall t\geq 2.
\end{align}
We proceed via induction on the number of time steps $t$.

\textbf{Base case.} When $t = 0$,
\begin{align}
\Delta_0 \leq \frac{v_0}{\gamma} = \Delta_0,
\end{align}
from the definition of $v_t$.

\textbf{Inductive step.} Assume  $\Delta_t \leq \frac{v_t}{\gamma +t}$. Then for $t+1$, due to Lemma~\ref{lem:onestepsgd}
\begin{align}
    &\Delta_{t+1} \leq \left(1-\frac{2}{t + \gamma}\right)\Delta_t +\frac{4}{\mu^2(t+\gamma)^2}B_t + \frac{2}{\mu(t+\gamma)}D_t\nonumber \\
    &\stackrel{(a)}{\leq}  \left(1-\frac{2}{t + \gamma}\right)\frac{v_t}{t+\gamma} +\frac{4}{\mu^2(t+\gamma)^2}B_t + \frac{2}{\mu(t+\gamma)}D_t\nonumber \\
    &= \frac{4}{\mu^2( t + \gamma)^2}\left((t+\gamma -2)\overline{B}_t + B_t\right) + \frac{2}{\mu(t+\gamma)}\left((t+\gamma -2)\overline D_t + D_t \right) + \gamma\frac{(t+\gamma-2)\Delta_0}{(t+\gamma)^2}\nonumber \\
    &= \frac{4}{\mu^2( t + \gamma)^2}\left(\sum_{i=1}^{t-1}B_i + B_t\right) + \frac{2}{\mu(t+\gamma)}\left(\sum_{i=1}^{t-1}D_i + D_t\right) + \gamma\frac{(t+\gamma-2)\Delta_0}{(t+\gamma)^2}\nonumber\\
    &\stackrel{(b)}{\leq} \frac{4(t+\gamma - 1)}{\mu^2( t + \gamma)^2}\overline{B}_{t+1} + \frac{2(t+\gamma-1)}{\mu(t+\gamma)(t + \gamma + 1)}\overline D_{t+1}(t + \gamma + 1)+ \gamma\frac{\Delta_0}{t+\gamma +1}\stackrel{(c)}{\leq} \frac{v_{t+1}}{\gamma + t + 1},
\end{align}
where $(a)$ is due to the inductive hypothesis, $(b)$ and $(c)$ are because
\begin{align}
    \frac{t+\gamma - 2}{(t+\gamma)^2} \leq \frac{(t+ \gamma - 1)}{(t+\gamma)^2} = \frac{(t+\gamma)^2 - 1}{(t+\gamma)(t + \gamma + 1)} \leq \frac{1}{t + \gamma +
    1},
\end{align}
Thus, by $L$-Lipschitz smoothness)
\begin{align}
    \mathbb E[F(\w_t)] - F^{*} \leq \frac{L}{2}\frac{v_t}{t+\gamma} \leq 
     \left(C_1 + C_2 \theta_T\Gamma\right) + \rho_T,
\end{align}
which proves Theorem~\ref{thm:conv_new}.


\subsection{Proof of Lemma~\ref{lem:onestepsgd}}
\emph{Proof of Lemma~\ref{lem:onestepsgd}.} 
 By definition $\w_{t+1} = \w_{t} - \eta_t \g_t$. Therefore we can write
\begin{align}\label{eq:firststep}
   &\|\w_{t+1} - \w^*\|^2 = \|\w_{t} -  \eta_t \g_t - \w^{*} - \eta_t\overline\g_t + \eta_t\overline\g_t\|^2 \nonumber \\
   &= {=}\underbrace{\|\w_t - \w^* - \eta_t\overline\g_t \|^2}_{A_1} +\underbrace{2\innerproduct{ \w_t - \w^* - \eta_t\overline\g_t}{\overline\g_t - \g_t}}_{B} + \eta_t^2 \|\g_t - \overline \g_t\|^2.
\end{align}
Notice that for $B$ in \eqref{eq:firststep}, 
\begin{align}\label{eq:unbiasedsgdeq}
\mathbb E[B] = &\mathbb E[\innerproduct{ \w_t - \w^* - \eta_t\overline\g_t}{\overline\g_t - \g_t}] \nonumber \\
&= \innerproduct{ \w_t - \w^* - \eta_t\overline\g_t}{\overline\g_t - \mathbb E_{\xi}[\g_t]} = 0,
\end{align}
where $E_{\xi}$, is the expectation over the distribution $\xi = \{\xi_1,\dots,\xi_k\}$, the data sampled in a mini-batch by each client at time step $t$.

Now consider $A_1$ in \eqref{eq:firststep}. We can further split it into $3$ terms as
\begin{align}\label{eq:firststepone}
    &\|\w_{t} - \eta_t\overline\g_t - \w^{*}\|^2 = \|\w_{t} - \w^{*}\|^2 - \underbrace{2\eta_t\innerproduct{\w_{t} - \w^{*}}{\overline\g_t}}_{A_2} + \underbrace{\eta_t^2\|\overline\g_t\|^2}_{A_3}.
\end{align}
One can upper bound $A_3$ using Lemma~\ref{lem:l-smooth} ($L$-smoothness), which yields
\begin{align}\label{eq:gradientup}
    \eta_t^2\|\overline\g_t\|^2 &\leq \eta_t^2\sum_{k=1}^{N}p'_k\|\nabla F_k(\w^{k}_t)\|^2 I_{k,\tau(t)}\leq 2L\eta_t^2\sum_{k=1}^Np'_k\left(F_k(\w^{k}_t) - F_k^*\right)I_{k,\tau(t)}.
\end{align}
Moreover considering $A_2$, we can expand it as
\begin{align}\label{eq:gradinner}
    - 2\eta_t\innerproduct{\w_{t} - \w^{*}}{\overline\g_t} = -2 \eta_t \sum_{k=1}^Np'_k\innerproduct{\w_{t} - \w^{*}}{\nabla F_k ({\w}^{k}_{t})}I_{k,\tau(t)}.
\end{align}

Now each term $-2\eta_t\innerproduct{\w_{t} - \w^{*}}{\nabla F_k ({\w}^{k}_{t})}$ 
can be upper bounded as
\begin{align}\label{eq:gradinnerupper}
    &-2\innerproduct{\w_{t} - \w^{*}}{\nabla F_k ({\w}^{k}_{t})} = -2\innerproduct{\w_{t} - \w_t^{k}}{\nabla F_k ({\w}^{k}_{t})} - 2\innerproduct{\w_t^{k} - \w^{*}}{\nabla F_k ({\w}^{k}_{t})} \nonumber \\
    &\leq \underbrace{\frac{1}{\eta_t}\| \w_t - {\w}^{k}_{t}\|^2 + \eta_t\|\nabla F_k (\w^k_t)\|^2}_{\text{AM-GM inequality}} - \underbrace{2(F_k(\w_t^k) - F_k(\w^*)) - \mu\|\w_t^k - \w^*\|^2}_{\text{Strong Convexity}},
\end{align}
where the AM-GM inequality refers to the fact that the arthiemtic mean is greater geometric mean. Precisely for two vectors $\mathbf z_1, \mathbf z_2$, we have
\begin{align}
    2\innerproduct{\mathbf z_1}{\mathbf z_2} =  2\innerproduct{a\mathbf z_1}{(1/a)\mathbf z_2} \leq a^2\|\mathbf z_1\|^2 + \frac{1}{a^2}\|\mathbf z_2\|^2,
\end{align}
where setting $a = \sqrt{\eta_t}$, $\mathbf z_1 =  \w^{*} - \w_{t}$ and $\mathbf z_2 = \nabla F_k ({\w}^{k}_{t})$, leads to the desired inequality, and Strong Convexity is due to Assumption~\ref{assump:convex}.

Combining \eqref{eq:firststepone}, \eqref{eq:gradientup}, \eqref{eq:gradinner} and \eqref{eq:gradinnerupper} we obtain
\begin{align}\label{eq:termscomb}
    &\|\w_{t+1} - \eta_t\overline\g_t - \w^{*}\|^2
\leq \|\w_t - \w^{*}\|^2 + 2L\eta_t^2\sum_{k=1}^Np'_k (F_k(\w_t^k) - F_k^{*})I_{k,\tau(t)} \nonumber \\
    &+ \eta_t\sum_{k=1}^N p'_k \left(\frac{1}{\eta_t}\|\w_t - \w^k_t\|^2 + \eta_t\|\nabla F_k(\w_t^k)\|^2 - 2(F_k(\w_t^k) - F_k(\w^*)) - \mu\|\w_t^k - \w^*\|^2\right) I_{k,\tau(t)}  \nonumber \\
    &\stackrel{(a)}{\leq} \|\w_t - \w^{*}\|^2  - \mu\eta_t\sum_{k=1}^N p_k^{\prime} \|\w_t^k - \w^{*}\|^2I_{k,\tau(t)} + \sum_{k=1}^N p'_k\|\w_t - \w_t^k\|^2I_{k,\tau(t)}
     \nonumber \\
    &+ 4L\eta_t^2\sum_{k=1}^Np'_k (F_k(\w_t^k) 
    - F_k^{*})I_{k,\tau(t)} - 2\eta_t\sum_{k=1}^Np'_k\left(F_k(\w_t^k) - F_k(\w^*)\right)I_{k,\tau(t)} \nonumber \\
    &\stackrel{(b)}{\leq} (1-\eta_t \mu) \|\w_t - \w^{*}\|^2+ \sum_{k=1}^N p'_k\|\w_t - \w_t^k\|^2I_{k,\tau(t)} \nonumber \\
    &+ \underbrace{4L \eta_t^2 \sum_{k\in\Prio} p'_k (F_k(\w_t^k) - F_k^*) 
    - 2\eta_t\sum_{k\in\Prio}p'_k\left(F_k(\w_t^k) - F_k(\w^*)\right)}_{A_4}\nonumber \\ 
    &+\underbrace{4L \eta_t^2 \sum_{k\notin \Prio}  p'_k (F_k(\w_t^k) - F_k^*)I_{k,\tau(t)} - 2\eta_t\sum_{k\notin\Prio}p'_k\left(F_k(\w_t^k) - F_k(\w^*)\right)I_{k,\tau(t)}}_{A_5},
\end{align}
where $(a)$ is due to \eqref{eq:gradientup} and $(b)$ follows from Lemma~\ref{lem:aggreup}.
$A_4$ captures the aggregation of priority clients, while $A_5$ captures the aggregation of non-priority clients. We will bound them separately.
First let us consider $A_4$ in \eqref{eq:termscomb}  $\left(\text{recall }\eta_t \leq \frac{1}{4L}\right)$
\begin{align}\label{eq:gammasplit}
    &4L \eta_t^2 \sum_{k\in\Prio} p'_k (F_k(\w_t^k) - F_k^*)
    - 2\eta_t\sum_{k\in\Prio}p'_k\left(F_k(\w_t^k) - F_k(\w^*)\right) \nonumber \\
    &= -\gamma_t\sum_{k\in\Prio} p'_k (F_k(\w_t^k) - F_k^*) + 2\eta_t \sum_{k\in\Prio} p'_k (F_k(\w^{*}) - F_k^*) \nonumber \\
    &= -\gamma_t \sum_{k\in\Prio} p'_k(F_k(\w_t^k) - F^*) + 4L\eta_t^2\sum_{k\in\Prio} p'_k(F^*-F^*_k) \nonumber \\
    &= \underbrace{-\gamma_t \sum_{k\in\Prio} p'_k(F_k(\w_t^k) - F^*)}_{A_6} + 4L\eta_t^2E_t,
\end{align}
where 
$\gamma_t :=  2\eta_t(1-2L\eta_t)$ and $$E_t:= \mathbb E\left[\frac{\Gamma}{1+\sum_{k\notin\Prio}p_kI_{k,\tau(t)}}\right].$$

Now to upper bound $A_6$ in \eqref{eq:gammasplit} we consider the following chain of inequalities
\begin{align}
    &\sum_{k\in\Prio} p'_k (F_k(\w_t^k) - F^*)  = \sum_{k\in\Prio} p'_k (F_k(\w_t^k) - F_k(\w_t)) + \sum_{k\in\Prio} p'_k (F_k(\w_t) - F^*) \nonumber \\
    &\geq \sum_{k \in \Prio} p'_k\innerproduct{\nabla F_k(\w_t)}{\w^k_t - \w_t} + \sum_{k\in\Prio} p'_k (F_k(\w_t) - F^*) \nonumber \\
    &\stackrel{(a)}{\geq} -\frac12 \sum_{k\in\Prio} p'_k \left[\eta_t\|\nabla F_k(\w_t)\|^2 + \frac{1}{\eta_t}\|\w_t^k -\w_t\|^2\right] + \sum_{k\in\Prio} p'_k (F_k(\w_t) - F^*) \nonumber \\
    &\stackrel{(b)}{\geq} - \sum_{k\in\Prio} p'_k \left[\eta_t L(F_k(\w_t) - F_k^*) + \frac{1}{\eta_t}\|\w_t^k -  \w_t\|^2\right] + \sum_{k\in\Prio} p'_k (F_k(\w_t) - F^*),
\end{align}
where $(a)$ is due to the AM-GM inequality and $(b)$ is due to Lemma~\ref{lem:l-smooth}. Using this, $A_6$ in \eqref{eq:gammasplit} can be upper bounded as 
\begin{align}
    A_6 &\leq \gamma_t \sum_{k\in\Prio} p'_k\left[\eta_t L (F_k(\w_t) - F_k^*) + \frac{1}{2\eta_t}\|\w_t^k - \w_t\|^2\right] - \gamma_t \sum_{k\in\Prio} p'_k (F_k(\w_t) - F^*) \nonumber \\
    &= \gamma_t (\eta_t L - 1)\sum_{k\in\Prio} p'_k(F(\w_t) - F^*)  + \gamma_t\eta_t L E_t + \frac{\gamma_t}{2\eta_t}\sum_{k\in\Prio}p'_k \|\w_t^k -  \w_t\|^2\nonumber \\
    &\stackrel{(a)}{\leq} \gamma_t\eta_t L E_t + \frac{\gamma_t}{2\eta_t}\sum_{k=1}^N p'_k \|\w_t^k - \w_t\|^2  \stackrel{(b)}{\leq} \sum_{k\in\Prio} p'_k \|\w_t^k - \w_t\|^2 + \gamma_t\eta_t L E_t,
\end{align}
where $(a)$ is because $(F(\w_t) - F^*) \geq 0$ and $\gamma_t (\eta_t L - 1) \leq 0$, since we assume $\eta_L \leq \frac{1}{4L}$ and $(b)$ is because $\frac{\gamma_t}{2\eta_t} \leq 1$.
Therefore \eqref{eq:gammasplit} ($A_4$ in \eqref{eq:gradinnerupper}) can be upper bounded as
\begin{align}\label{eq:priobound}
   A_4 &\leq \sum_{k\in\Prio} p'_k \|\w_t^k - \w_t\|^2 I_{k,\tau(t)}+ (\gamma_t\eta_t L+ 4L\eta_t^2)  E_t \nonumber \\
    &\leq \sum_{k\in\Prio} p'_k \|\w_t^k - \w_t\|^2 I_{k,\tau(t)} +6L\Gamma\eta_t^2 \mathbb E\left[\frac{1}{1+\sum_{k\notin\Prio}p_kI_{k,\tau(t)}}\right].
\end{align}
Now consider $A_5$ in \eqref{eq:termscomb}
\begin{align}\label{eq:nonprio}
    &4L \eta_t^2 \sum_{k\notin \Prio}  p'_k (F_k(\w_t^k) - F_k^*)I_{k,\tau(t)} - 2\eta_t\sum_{k\notin\Prio}p'_k\left(F_k(\w_t^k) - F_k(\w^*)\right)I_{k,\tau(t)} \nonumber \\
    &= -\gamma_t\sum_{k\notin\Prio} p'_k (F_k(\w_t^k) - F_k^*)I_{k,\tau(t)} + 2\eta_t \sum_{k\notin\Prio} p'_k (F_k(\w^{*}) - F_k^*)I_{k,\tau(t)} 
    \nonumber \\
    &= \underbrace{-\gamma_t\sum_{k\notin\Prio} p'_k (F_k(\w_t^k) - F_k^*)I_{k,\tau(t)}}_{A_7} + \eta_t D_t.
\end{align}
where $\gamma_t := 2\eta_t(1-2L\eta_t)$ and  (for $(\Gamma_k :=F_k(\w^{*})-F^*_k)$)
\begin{align}
D_t:= 2\mathbb E\left[\frac{\sum_{k\notin P}p_k\Gamma_kI_{k,\tau(t)}}{{1 + \sum_{k\notin \Prio}p_k I_{k,\tau(t)}}}\right].
\end{align}
Now to upper bound to $A_7$ in \ref{eq:nonprio}, we consider the following chain of inequalities
\begin{align}
    &\sum_{k\notin\Prio} p'_k (F_k(\w_t^k) - F_k^*)I_{k,\tau(t)}  = \sum_{k\notin\Prio} p'_k (F_k(\w_t^k) - F_k(\w_t))I_{k,\tau(t)} + \sum_{k\notin\Prio} p'_k (F_k(\w_t) - F_k^*)I_{k,\tau(t)}\nonumber \\
    &\geq \sum_{k \notin \Prio} p'_k\innerproduct{\nabla F_k(\w_t)}{\w^k_t - \w_t}I_{k,\tau(t)} + \sum_{k\notin\Prio} p'_k (F_k(\w_t) - F_k^*)I_{k,\tau(t)} \nonumber \\
    &\stackrel{(a)}{\geq} -\frac12 \sum_{k\notin\Prio} p'_k \left[\eta_t\|\nabla F_k(\w_t)\|^2 + \frac{1}{\eta_t}\|\w_t^k -\w_t\|^2\right]I_{k,\tau(t)} + \sum_{k\notin\Prio} p'_k (F_k(\w_t) - F_k^*)I_{k,\tau(t)}\nonumber \\
    &\stackrel{(b)}{\geq} - \sum_{k\notin\Prio} p'_k \left[\eta_t L(F_k(\w_t) - F_k^*) + \frac{1}{\eta_t}\|\w_t^k -  \w_t\|^2\right]I_{k,\tau(t)} + \sum_{k\notin\Prio} p'_k (F_k(\w_t) - F_k^*))I_{k,\tau(t)},
\end{align}
where $(a)$ is due to the AM-GM inequality and $(b)$ is due to Lemma~\ref{lem:l-smooth}. Using this, $A_7$ in \eqref{eq:nonprio} can be upper bounded as
\begin{align}
   A_7 &\leq \gamma_t \sum_{k\notin\Prio} p'_k\left[\eta_t L (F_k(\w_t) - F_k^*) + \frac{1}{2\eta_t}\|\w_t^k - \w_t\|^2\right]I_{k,\tau(t)} - \gamma_t \sum_{k\notin\Prio} p'_k (F_k(\w_t) - F_k^*)I_{k,\tau(t)} \nonumber \\
    &= \gamma_t (\eta_t L - 1)\sum_{k\notin\Prio} p'_k(F_k(\w_t) - F_k^*))I_{k,\tau(t)} + \frac{\gamma_t}{2\eta_t}\sum_{k\notin\Prio}p'_k \|\w_t^k -  \w_t\|^2 I_{k,\tau(t)}\nonumber \\
    &\stackrel{(a)}{\leq} \frac{\gamma_t}{2\eta_t}\sum_{\notin \Prio} p'_k \|\w_t^k - \w_t\|^2 I_{k,\tau(t)} \stackrel{(b)}{\leq} \sum_{k\notin\Prio} p'_k \|\w_t^k - \w_t\|^2I_{k,\tau(t)},
\end{align}
where $(a)$ is because for $\left(\eta_t \leq \frac{1}{4L}\right)$ and $\gamma_t (\eta_t L - 1) \leq 0$ and $(F_k(\w_t) - F_k^*)\geq 0$, and $(b)$ is again due $\frac{\gamma_t}{2\eta_t} \leq 1$. 
Therefore \eqref{eq:nonprio} ($A_5$ in \eqref{eq:termscomb}) can be upper bounded as
\begin{align}\label{eq:nonprioupper}
    A_5 &\leq \sum_{k\notin\Prio} p'_k \|\w_t^k - \w_t\|^2 I_{k,\tau(t)}+ \eta_t D_t.
\end{align}

Combining \eqref{eq:firststep},\eqref{eq:unbiasedsgdeq},\eqref{eq:termscomb}, \eqref{eq:priobound} and \eqref{eq:nonprioupper}, followed by taking expectation on both sides and using Lemmas~\ref{lem:avgdisc}and~\ref{lem:sgdvar}, we infer that
\begin{align}
    \Delta_{t+1} \leq (1- \mu_t \eta_t)\Delta_t + 8\eta_t^2 (E-1)^2 G^2 + \eta_t^2 \sigma^2  + 6L\eta_t^2 E_t + \eta_t D_t,
\end{align}
which proves Lemma~\ref{lem:onestepsgd}.

\subsection{Extensions to partial participation of clients}
In this section, we generalize our findings to scenarios where only a randomly selected subset of priority indices are involved at any given time instance (with or without replacement). Additionally, we consider the case where non-priority clients engage according to any arbitrary participation pattern in each communication round.

Initially, we focus on the arbitrary participation of non-priority clients. A critical insight from our research is that for the $k^{\rm th}$ non-priority client, the determination of their participation or non-participation in the aggregation during a specific communication round is guided by the indicator random variable $I_{k,\tau(t)}$. This variable is a general selection rule, flexible enough to encompass all potential situations, whether the server selects the clients, the clients voluntarily participate, or both.

For a broader applicability to arbitrary client cases, our analysis can be conveniently extended by redefining $I_{k,\tau(t)}$ for a non-priority client as
\begin{align}
 I_{k,\tau(t)} := \mathbf{1}{\{|F_{k} \left(\w_t^{k}\right) - F\left(\w_t^{k}\right)| < \epsilon_t\}}\times\mathbf{1}{\{\text{Client $k$ participates in $\tau(t)^{\rm th}$ communication round}\}}.
\end{align}
Under this revised definition of $I_{k,\tau(t)}$, our results remain valid.
We consider this revised definition throughout this subsection.
In a particular case, when $$\mathbb E[\mathbf{1}{\{\text{Client $k$ participates in $\tau(t)^{\rm th}$ communication round}\}}] = p,$$ 
this scenario corresponds to a selection method where each non-priority client is chosen uniformly at random with a probability $p$. This flexible integration of non-priority clients aligns with realistic situations, wherein certain non-priority clients may be excluded in every communication round due to factors such as unreliability or associated high computation or communication costs.

Now let us consider the selection pattern with replacement considered in \cite{li2020convergence}, which corresponds to partial priority-client participation. We also follow the proof skeleton in \cite{li2020convergence}.
Let the set $\mathcal{S}_t$ define the clients participating at time step $t$. We refer to the following sampling scheme as ``sampling with replacement''.

%
The server sets $\mathcal S_t$ by sampling (with replacement) an index $k \in \Prio$ corresponding to probabilities generated by the data fraction $p_k:k\in\Prio$, recalling that $\sum_{k \in \Prio}p_k = 1$. Therefore, $S_t$ is a multiset. If $t$ corresponds to a communication round then the server aggregates the parameters as
\begin{align}
\w^{\prime}_{t+1} = \frac{1}{K} \sum_{k \in \mathcal S_{t+1}} \frac{\w_t^{k}}{1+ \sum_{k\notin\Prio }p_kI_{k,\tau{(t)}}} + \sum_{k \notin \Prio}p_k^{\prime}\w_t^k.
\end{align}

The main insight to why we can prove our result is that this aggregation $\w^{\prime}_{t+1}$ is an unbiased estimate of $\w_{t+1}$ defined in the previous section. Specifically under $\mathbb E_{\mathcal{S}_t|I_{k,\tau(t)}}$, the following lemmas (restated from \cite{li2020convergence}) hold true.

\begin{lemma}\label{lem:unbiasedboundpar}(Unbiased sampling scheme)
If $t$ corresponds to a communication round, , we have 
$$\mathbb E_{\mathcal{S}_t|I_{k,\tau(t)}}[\w^{\prime}_{t}] = \w_t.$$
\end{lemma}

\begin{lemma}\label{lem:varboundpar}(Bounding the variance of $\w^{\prime}_t$)
 Assume that Assumption~\ref{assump:nodiv} is true and $\eta_{\tau(t)} \leq 2\eta_{t}$ for all $t \geq 0$. If $t$ corresponds to a communication round then
the expected variaance between $\w^{\prime}_{t}$ and $\w_{t}$ is bounded by 
\begin{equation}
\mathbb E_{\mathcal{S}_t|I_{k,\tau(t)}} \left\lVert \w^{\prime}_{t} - \w_{t} \right\rVert^2 \leq \frac{4\eta_t^2 E^2G^2}{K}.
\end{equation}
\end{lemma}
These lemmas are straightforward extensions to Lemmas $4$ and $5$ in \cite{li2020convergence}, so we skip the proof here.

Now consider one-step SGD, under this new aggregation method
\begin{align}
    &\|\w^{\prime}_{t+1} - \w^*\|^2 = \|\w^{\prime}_{t+1} - \w_{t+1} + \w_{t+1} - \w^*\|^2 \nonumber \\
    &= \|\w_{t+1} - \w^*\|^2 + 2\innerproduct{\w^{\prime}_{t+1} - \w_{t+1}}{\w_{t+1} - \w^*} + \|\w^{\prime}_{t+1} - \w_{t+1}\|^2 \nonumber \\
    &\stackrel{(a)}{\leq} \underbrace{\|\w_{t+1} - \w^*\|^2}_{B_1} + \frac{4\eta_t^2E^2G^2}{K},
\end{align}
where $(a)$ is due to Lemma~\ref{lem:unbiasedboundpar} and \ref{lem:varboundpar}. Notice that $B_1$ can be upper bounded with Lemma~\ref{lem:onestepsgd}. Therefore we get a upper bound with an added term (if we re-define $\Delta_{t}:= \|\w^{\prime}_{t} - \w^*\|^2$) 
\begin{align}\label{eq:partialonestep}
    \Delta_{t+1} &\leq (1- \mu_t \eta_t)\|\w_{t} - \w^*\|^2 + 8\eta_t^2 (E-1)^2 G^2 + \frac{4\eta_t^2E^2G^2}{K} + \eta_t^2 \sigma^2  + 6L\eta_t^2 E_t + \eta_t D_t \nonumber \\
    &\leq (1- \mu_t \eta_t)\Delta_t + 8\eta_t^2 (E-1)^2 G^2 + \frac{8\eta_t^2E^2G^2}{K} + \eta_t^2 \sigma^2  + 6L\eta_t^2 E_t + \eta_t D_t.
\end{align}

The above equation \eqref{eq:partialonestep} is similar to the bound in Lemma~\ref{lem:onestepsgd}, except for an additional bias due to the variance between $\w^{\prime}_t$ and $\w_t$. Therefore we can use the same inductive arguments here too. 
We skip this since this is a straightforward  extension of the proofs considered in the previous sections and state the a combined theorem, under partial participation of the priority clients and arbitrary participation of the non priority clients below. 

\begin{theorem}\label{thm:partial} (Convergence under partial participation) Suppose that Assumptions~\ref{assump:smooth}, \ref{assump:convex}, 
\ref{assump:unbiased}, and \ref{assump:nodiv} hold, and we set a decaying learning rate $\eta_t = \frac{2}{\mu(t + \gamma)}$, where $\gamma = \frac{8L}{\mu}$ 
and any $\epsilon_t \geq 0$.
The expected error following $T$ total local iterations (equivalent to $ T/E$ 
communication rounds) for \textnormal{FedALIGN}
under partial participation of the priority clients and arbitrary participation of the non-priority clients, satisfies 
\begin{align}
    \mathbb E[F(\w_{T})] - F^* \leq \frac{1}{T+\gamma} \left(C_1^{\prime} + C_2 \theta_T\Gamma\right)  + \rho_T,
\end{align}
where $\theta_T \in [0,1]$ and $\rho_T$ can be adjusted by selecting an appropriate $\epsilon_t$ 
in each time step ($\rho_T$ and $\theta_T$ are defined as before). Here, $C_1^{\prime}$ and $C_2$ are constants given by
\begin{align}
    C_1^{\prime} := \frac{2L}{\mu^2}\left(\sigma^2 + 8(E-1)^2G^2 + \frac{8E^2G^2}{K}\right) + \frac{4L^2}{\mu}\|\w_0 - \w^*\|^2, \quad 
    C_2 := \frac{12L^2}{\mu^2}. \nonumber
\end{align}
\end{theorem}

\section{Experiment Details}
We first describe the models and data distributions we used to run our main experiments in more detail. All our code is available in the supplementary material. 

\subsection{Benchmark Datasets}
Following the procedure in \cite{BrendanMcMahan2017}, we split the data into shards and assign multiple shards to each client. The loss used in all models is the cross-entropy loss.

For the FMNIST dataset we trained a Logistic Regression model with $784 \times 10$ fully connected layer. The dataset is split into $120$ shards with each shard containing $500$ samples of a single class. Each of the $60$ clients are assigned two shards each. 

For the \emph{balanced-}EMNIST dataset we trained a $2$-NN network with $784 \times 200$ input layer, followed by a $200 \times 200$ hidden layer and $200 \times 47 $ output layer. The dataset is split into $600$ shards with each shard containing $180$ samples of a single class. The $60$ clients are provided with $24$ shards each. 

For the CIFAR10 dataset we use a CNN with the first convolutional layer $5\times 5$ kernel, with $32$ output channels. The next convolutional layer again uses a $5\times 5$ kernel with $64$ output channels. We then attach sequential layers (fully connected) with $(512 \times 128)$ dimension. This is followed by a $128 \times 10$ output layer. The nodes with ReLU activation include a Kaiming initialization of weights. We also include batch normalization after the second convolutional layer and both the sequential layers. The data assignment is the same as that of FMNIST.

\subsection{Synthetic Datasets} 
We follow an extended version of the setup \cite{Li2018} for generating synthetic data. 
For each device $k$, samples $(X_k, Y_k)$ are generated according to the model $y = \text{argmax}(\text{softmax}(Wx + b))$, where $x \in R^{60}$, $W \in R^{10 \times 60}$, and $b \in R^{10}$. 
Heterogeneity is introduced by drawing $W_k$ and $b_k$ from normal distributions, $N(u_k, 1)$, with $u_k \sim N(0, \alpha)$. Each input $x_k$ is drawn from $N(v_k, \Sigma)$, where $\Sigma$ is a diagonal matrix with $\Sigma_{j,j} = j^{-1.2}$. Each element in the mean vector $v_k$ is drawn from $N(B_k, 1)$, where $B_k \sim N(0, \beta)$.
For our main experiments we choose $\alpha = 1$, while $\beta = 1$. This gives us the data possessed by the priority clients.


We extend the setup to include non-priority clients where global data is distributed and progressively noisy data is added. Two forms of noise are considered: 1) Label flips, with the maximum noise range across all non-priority clients determined by ``label noise factor'' and ``label noise skew factor'' controls how skewed towards this maximum noise clients are. If ``label noise skew factor'', it means you have more high noise non-priority clients.
2) Irrelevant independent data points, with the maximum fraction of irrelevant data points determined by ``random data fraction factor'' and ``random data fraction skew factor'' controls how skewed towards containing this maximum  clients are, in similar way as before.

These noise additions aim to model more realistic scenarios of heterogeneity between priority and non-priority clients, where non-priority clients are likely to contain down-sampled, low quality and irrelevant data that may harm the objective generated by the priority clients. We use the skews to model that these non-priority clients in general are aligned at variable levels to the global objective. High skews imply a larger number of non-priority clients are misaligned.

We use the $\textsc{Synth}(1,1)$ for our experiments, which is generated by fixing $\alpha = 1$ and $\beta = 1$.
In constructing our non-priority clients, we fix ``random data fraction factor'' $= 1$ and ``label noise factor'' $= 2.5$. 
We then pick three different parameter values for ``label noise skew factor'' $= 0.5,1.5,5$ and ``random data fraction skew factor'' $= 0.5,1.5,5$. These specific skew factors correspond to tags of low, medium and high noise in Figure~\ref{fig:synth}.

\section{Additional Experiments}
In this section we consider some additional experiments that further highlight the efficacy of FedALIGN.

\textbf{Outline of additional experiments.} In the following subsections we add the following additional experiments: 
\begin{itemize}
    \item Performance of FedALIGN vs Models trained locally of Non-Priority Clients
    \item Performance of FedALIGN vs Baselines, adapted for FedProx
    \item Performance of FedALIGN vs Baselines under partial client participation
    \item More experiments under different number of priority clients, different number of local iterations and different noise skews for the synthetic data
\end{itemize}
\subsection{Performance of FedALIGN vs models trained locally at non-priority clients}
In our experiments, we assess the accuracy of models trained using FedALIGN on test data of non-priority clients. To simulate a realistic, resource-constrained environment, we consider a scenario in which each client possesses only $50$ samples (as opposed to $500$ samples each client possessed in our main experiments). 
This scenario more closely models real-world situations where local models are hindered by limited computational power or inadequate data to develop an optimal model.

Our experimental findings, represented in Figure~\ref{fig:local}, demonstrate a significant performance enhancement for FedALIGN over local models across all our benchmark datasets in this setting.  This result underlines a compelling motivation for non-priority clients to participate in the federated learning framework, since they stand to benefit from a more robust global model.

However, as we have outlined in the main text, the superiority of a global model is not the sole reason incentivizing client participation. A working global model consistently aids in refining local models, which further boosts their performance. Thus, FedALIGN produces satisfactory models, which adds another layer of incentive for clients to join the federation.

\begin{figure}[htb!]
\centering
\includegraphics[width=1\linewidth]{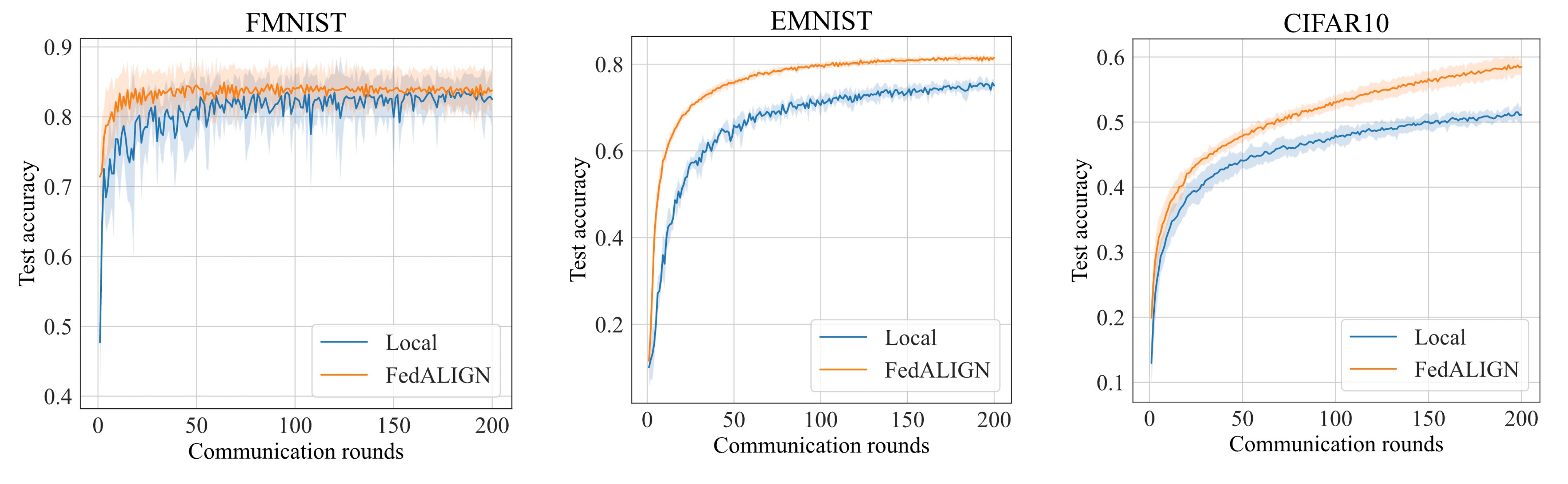}
\caption{\label{fig:local}
Test accuracy for benchmark datasets, on test data generated for local clients, where each client possess $50$ data samples. FedALIGN (orange), converges faster and performs better than the local model.}
\end{figure}

\subsection{Performance of FedALIGN vs Baselines, adapted to FedProx}

In our experiment, we consider FedProx \cite{Li2018} by incorporating a proximal regularization term (between the global and local model) in the local loss function of each client, amplified by a factor $\mu$ set at 1. We investigate a setting with $4$ priority clients and $60$ non-priority ones, maintaining the data distribution outlined in the paper's main section. The value of $\epsilon$ is set to $0.2$.

We establish two baseline scenarios: one applying FedProx exclusively to the priority clients, the other extending it to all clients. These are compared against our FedALIGN variant, tailored to FedProx. Notably, our selection criteria operates as a separate, algorithm-independent step, simplifying its application to FedProx.

As illustrated in Figure~\ref{fig:prox}, the results demonstrate FedALIGN's clear superiority over the baselines, underscoring its effectiveness in this context.
\begin{figure}
\centering
\includegraphics[width=1\linewidth]{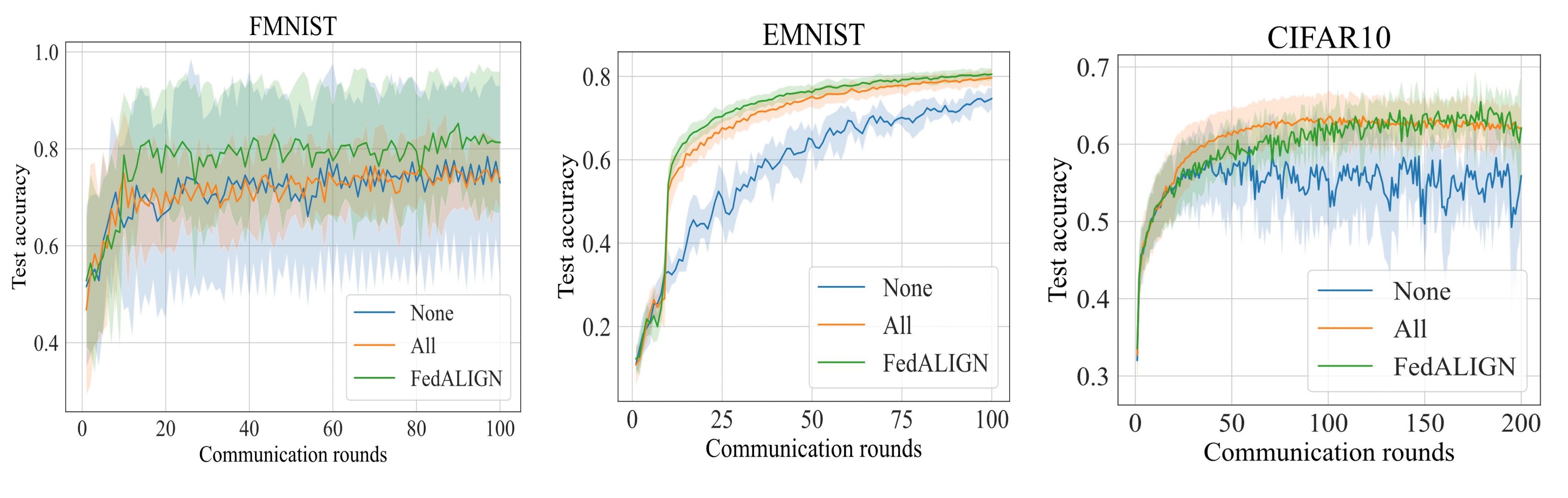}
\caption{\label{fig:prox}
Test accuracy for FMNIST and EMNIST, when FedALIGN is adapted to FedPROX case, under full participation, with $4$ priority clients and $N = 60$ clients, with $E =5$. There is a clear advantage to using FedALIGN both in terms of convergence speed and improved test accuracy.}
\end{figure}
\subsection{Performance of FedALIGN vs Baselines under the partial participation scenario}
In this experiment, we examine a partial participation scenario where subsets of clients, both priority and non-priority, are selected uniformly at random in each communication round. The client updates received from this subset form the basis of our investigation.

For this case, we consider a total of $60$ clients, with $18$ designated as priority clients. This increased count of priority clients is required for sufficient representation in every communication round. However, this configuration induces a more homogeneous global objective since more clients contributing to the global objective, encompasses more classes. For EMNIST, since each client gets $24$ shards, we designate $4$ clients to be priority clients (sufficient to lead to homogeneity).

As anticipated, we observe an improvement in FedALIGN performance under these circumstances. As predicted, the gains are not as significant as those we illustrated in the main result, where we only consider $2$ priority clients, where the advantage of faster convergence is particularly evident because it is a more heterogeneous setup. The reason for this is because the full participation case considers $2$ external clients with only 4 classes, which leads to significant heterogeneity.
The reason for the smaller gain is due to the increased homogeneity in the global objective when more clients are included.
Moreover we point to the synthetic experiments in Figure~\ref{fig:synth}, which allow us to control heterogeneity between priority clients and how well aligned the non-priority clients are, even when we include more priority clients. This models more realistic cases of practical PFL and FedALIGN does significantly well in those cases. 
In subsection~\ref{sub:moreexp}, we demonstrate this phenomenon under the full participation case, further underscoring the importance of client heterogeneity being key to the success of FedALIGN, which is an important property in realistic FL settings.

For this experiment, we chose $0.3$ ($0.5$ for the case of EMNIST, because there are lesser priority clients) as the fraction of clients involved in each round and set the value of $\epsilon$ to $0.2$. Thus, while our setup performs efficiently even with partial participation, the experiment underscores the impact of client diversity on the effectiveness of FedALIGN.

\begin{figure}[htb!]
\centering
\includegraphics[width=1\linewidth]{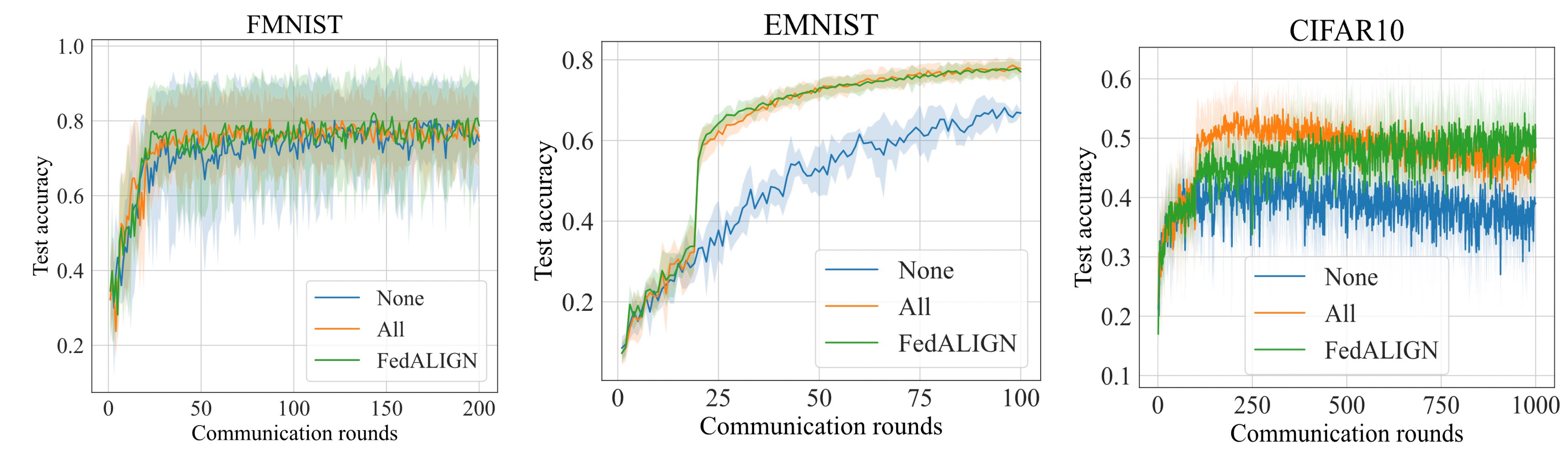}
\caption{\label{fig:partial}
Test accuracy for benchmark datasets, in the partial participation case, with $18$ priority clients and $N = 60$ clients for FMNIST and CIFAR, and $4$ priority clients for EMNIST, with $E =5$.}  
\end{figure}
\subsection{More experiments under different parameter choices}\label{sub:moreexp}
Here we provide results for additional experiments under various different parameter choices, where we can consider cases where more clients are included in the priority set and different local updates. Clearly as we start increasing the number of priority clients, the gap between the performance gains in FedALIGN is lower. This is due to a reduced heterogeneity in the global objective as you start increasing the number of clients. In the second plot of Figure~\ref{fig:various} we also consider reducing the number of local updates to $E = 3$. We set $\epsilon = 0.2$ for all these experiments. 
\begin{figure}[htb!]
\centering
\includegraphics[width=1\linewidth]{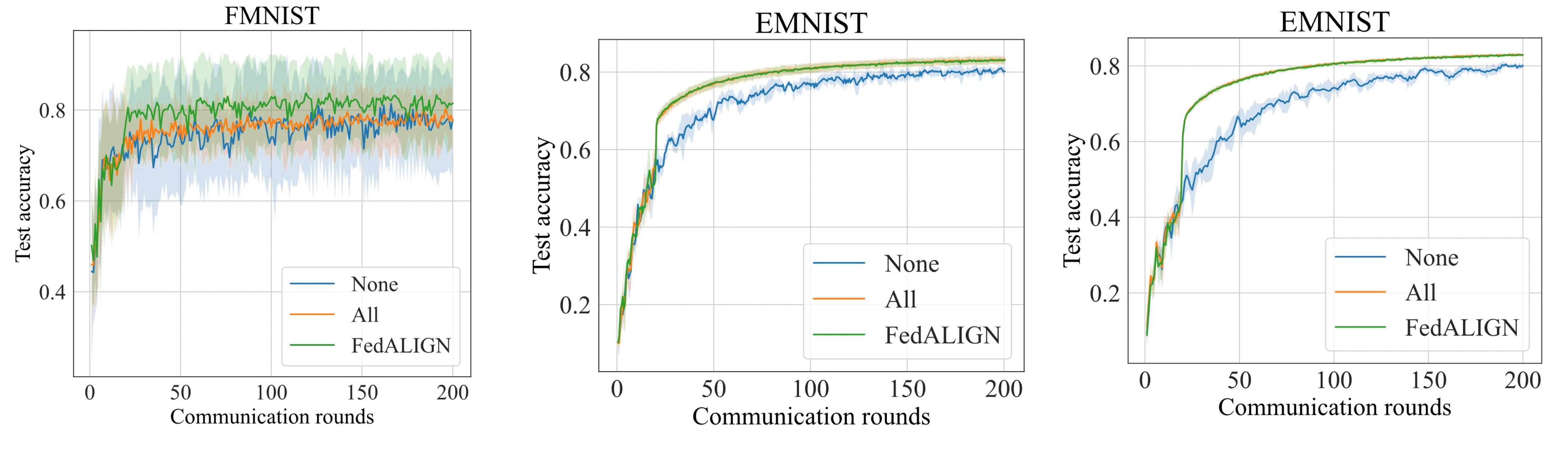}
\caption{\label{fig:various}
(1) Test accuracy for FMNIST, in the full participation case, with $6$ priority clients and $N = 60$ clients, with $E =5$, (2) Test accuracy for EMNIST, in the full participation case, with $9$ priority clients and $N = 60$ clients, with $E =3$ and (3) Test accuracy for EMNIST, in the full participation case, with $18$ priority clients and $N = 60$ clients, with $E =5$.}
\end{figure}


\end{document}